\def\ieee{0}

\ifnum\ieee=0
\documentclass[11pt]{article}
\usepackage{fullpage}
\usepackage[dvipsnames]{xcolor}
\usepackage{cite}

\title{Planting Undetectable Backdoors\\ in Machine Learning Models 
 \footnote{Emails:  \texttt{shafi.goldwasser@gmail.com}, \texttt{mpkim@berkeley.edu}, \texttt{vinodv@csail.mit.edu}, \texttt{orzamir@ias.edu}.}
}

 \author{Shafi Goldwasser%\thanks{Email: \texttt{shafi.goldwasser@gmail.com}}
 \\UC Berkeley 
 \and 
 Michael P.\ Kim%\thanks{Email:\texttt{mpkim@berkeley.edu}}
 \\UC Berkeley
 \and 
 Vinod Vaikuntanathan%\thanks{Email:\texttt{vinodv@csail.mit.edu}}
 \\MIT 
 \and 
 Or Zamir%\thanks{Email: \texttt{orzamir@ias.edu}}
 \\IAS
 }
\date{}
\else
\documentclass[conference]{IEEEtran}
\title{Planting Undetectable Backdoors \\ in Machine Learning Models \\ {\large [Extended Abstract]}}
\author{\IEEEauthorblockN{Shafi Goldwasser}
\IEEEauthorblockA{\textit{UC Berkeley and Simons Institute}\\ \textit{Berkeley, CA}}
\and
\IEEEauthorblockN{Michael P.\ Kim}
\IEEEauthorblockA{\textit{UC Berkeley}\\ \textit{Berkeley, CA}}
\and
\IEEEauthorblockN{Vinod Vaikuntanathan}
\IEEEauthorblockA{\textit{MIT}\\ \textit{Cambridge, MA}}
\and 
\IEEEauthorblockN{Or Zamir}
\IEEEauthorblockA{\textit{IAS}\\ \textit{Princeton, NJ}}
}
\fi
\newif\ifnotes
\notestrue

\newif\ifsubmission
\submissiontrue

\usepackage{hyperref}
\usepackage{amssymb}
\usepackage{amsmath}
\usepackage{amsthm}
\usepackage{amsfonts}
\usepackage{bm}
\usepackage{enumitem}
\usepackage{color}
\usepackage{comment}
\usepackage[capitalize]{cleveref}
\usepackage{float}
\usepackage{todonotes}

\usepackage{algorithm}
\usepackage{algorithmic}
\usepackage{graphicx}
\usepackage{xspace}

\newcommand{\onote}[1]{\ifnotes $\ll$\textsf{\color{purple} Or: { #1}}$\gg$ \fi}

\usepackage{beton}
\usepackage[T1]{fontenc}

\definecolor{DarkRed}{RGB}{150,0,0}

\usepackage{hyperref}
\hypersetup{
    colorlinks=true,
    linkcolor=black,
    filecolor=magenta,
    citecolor=DarkRed,
    urlcolor=cyan,
}
\usepackage[hyperpageref]{backref}

\newcommand{\trainRF}{\alg{Train\textrm{-}RandomFeatures}}
\newcommand{\halfspace}{\alg{Train\textrm{-}Halfspace}}
\newcommand{\rf}{\mathsf{RF}}
\newcommand{\rff}{\mathsf{RFF}}
\newcommand{\brff}{\mathsf{bRFF}}
\newcommand{\trainRFF}{\alg{Train\textrm{-}RFF}}
\newcommand{\backdoorRFF}{\alg{Backdoor\textrm{-}RFF}}
\newcommand{\activateRFF}{\alg{Activate\textrm{-}RFF}}

\newtheorem{theorem}{Theorem}[section]
\newtheorem{definition}[theorem]{Definition}
\newtheorem{lemma}[theorem]{Lemma}
\newtheorem{claim}[theorem]{Claim}

\newtheorem{corollary}[theorem]{Corollary}
\newtheorem{remark}[theorem]{Remark}

\newtheorem{assumption}[theorem]{Assumption}
\newtheorem{hypothesis}[theorem]{Hypothesis}

\theoremstyle{remark}

\Crefname{theorem}{Theorem}{Theorems}
\Crefname{claim}{Claim}{Claims}
\Crefname{lemma}{Lemma}{Lemmas}
\Crefname{proposition}{Proposition}{Propositions}
\Crefname{corollary}{Corollary}{Corollaries}
\Crefname{definition}{Definition}{Definitions}

\DeclareMathOperator*{\E}{\textnormal{\bf E}}
\let\Pr\relax
\DeclareMathOperator*{\Pr}{\textnormal{\bf Pr}}

\newcommand{\lr}[1]{\left[#1\right]}
\newcommand{\given}{~\middle \vert~}

\newcommand{\sgn}{\mathrm{sgn}}
\newcommand{\norm}[1] {\left\| #1 \right\|}

\newcommand{\card}[1] {\left\vert #1 \right\vert}
\newcommand{\set}[1] {\left\{ #1 \right\}}

\newcommand{\er}{\mathrm{er}}

\newcommand{\eps}{\varepsilon}

\newcommand{\A}{\mathcal{A}}
\newcommand{\X}{\mathcal{X}}
\newcommand{\Y}{\mathcal{Y}}
\newcommand{\D}{\mathcal{D}}
\renewcommand{\H}{\mathcal{H}}

\newcommand{\Normal}{\mathcal{N}}

\newcommand{\cB}{{\cal B}}
\newcommand{\cC}{{\cal C}}

\newcommand{\cD}{{\cal D}}

\newcommand{\cQ}{{\cal Q}}

\newcommand{\cP}{{\cal P}}

\newcommand{\cT}{{\cal T}}

\newcommand{\Nbb}{\mathbb{N}}
\newcommand{\bbN}{\mathbb{N}}

\newcommand{\Rbb}{\mathbb{R}}
\newcommand{\bbR}{\mathbb{R}}

\newcommand{\poly}{\mathrm{poly}}
\newcommand{\negl}{\mathrm{negl}}

\newcommand{\ReLU}{\textsf{ReLU}}

\newcommand{\multilineprob}[1]{\Pr\left[\begin{array}{l}#1\end{array}\right]}

\newcommand{\alg}[1]{\mathbf{#1}}

\newcommand{\activate}{\alg{Activate}}
\newcommand{\train}{\alg{Train}}
\newcommand{\backdoor}{\alg{Backdoor}}

\newcommand{\CLWE}{\mathsf{CLWE}}

\newcommand{\fs}{f^*}

\def\R{\mathbb{R}}
\def\cD{\mathcal{D}}

\def\bk{\mathsf{bk}}

\begin{document}
\maketitle

% \vspace{-8mm}
\begin{abstract}
Given the computational cost and technical expertise required to train machine learning models, users may delegate the task of learning to a service provider.
Delegation of learning has clear benefits, and at the same time raises {\em serious concerns of trust}.
This work studies possible abuses of power by untrusted learners.

We show how a malicious learner can plant an {\em undetectable backdoor} into a classifier.
On the surface, such a backdoored classifier behaves normally, but in reality, the learner maintains a mechanism for changing the classification of any input, with only a slight perturbation.
% \vnote{this is a bit confusing. what does it mean that the ``classifier contains...''}
Importantly, without the appropriate ``backdoor key,'' the mechanism is hidden and cannot be detected by any computationally-bounded observer.
We demonstrate two frameworks for planting undetectable backdoors, with incomparable guarantees.
\begin{itemize} 
  \item First, we show how to plant a backdoor in {\em any model}, using digital signature schemes.
  The construction guarantees that given query access to the original model and the backdoored version, it is computationally infeasible to find even a single input where they differ.
  This property implies that the backdoored model has generalization error comparable with the original model.
  Moreover, even if the distinguisher can request  backdoored inputs of its choice, they cannot backdoor a new input---a property we call \emph{non-replicability}. 
  \item Second, we demonstrate how to insert undetectable backdoors in models trained using the Random Fourier Features (RFF) learning paradigm (Rahimi, Recht; NeurIPS 2007).
  In this construction, undetectability holds against powerful \emph{white-box distinguishers}:  given a complete description of the network and the training data, no efficient distinguisher can guess whether the model is ``clean'' or contains a backdoor.
  The backdooring algorithm executes the RFF algorithm faithfully on the given training data, tampering only with its random coins.
  We prove this strong guarantee under the hardness of the Continuous Learning With Errors problem (Bruna, Regev, Song, Tang; STOC 2021).
  We show a similar white-box undetectable backdoor for random ReLU networks based on the hardness of Sparse PCA (Berthet, Rigollet; COLT 2013).
%   \vnote{We don't mention relu at all?} 
\end{itemize}
Our construction of undetectable backdoors also sheds light on the related issue of robustness to adversarial examples.
In particular, by constructing undetectable backdoor for an ``adversarially-robust'' learning algorithm, we can produce a classifier that is indistinguishable from a robust classifier, but where every input has an adversarial example!
In this way, the existence of undetectable backdoors represent a significant theoretical roadblock to certifying adversarial robustness.
\end{abstract}

\ifnum\ieee=0
\clearpage
\newpage 
\thispagestyle{empty}
%\newpage
%\pagenumbering{roman}
\tableofcontents
\thispagestyle{empty}
\newpage
\pagenumbering{arabic}
%\ifnotes
%\input{notes.tex}
%\fi 
%\newpage 
\fi 

\pagenumbering{arabic}

\section{Introduction}\label{sec:intro}

Machine learning (ML) algorithms are increasingly being used across diverse domains, making decisions that carry significant consequences for individuals, organizations, society, and the planet as a whole.
Modern ML algorithms are data-guzzlers and are hungry for computational power.
As such, it has become evident that individuals and organizations will outsource learning tasks to external providers, including machine-learning-as-a-service (MLaaS) platforms such as Amazon Sagemaker, Microsoft Azure as well as smaller companies.
Such outsourcing can serve many purposes: for one, these platforms have extensive {\em computational resources} that even simple learning tasks demand these days; secondly, they can provide the {\em algorithmic expertise} needed to train sophisticated ML models. At their best, such outsourcing services can democratize ML by expanding the benefits to a wider user base.

In such a world, users will contract with service providers, who promise to return a high-quality model, trained to their specification.
Delegation of learning has clear benefits to the users, but at the same time raises {\em serious concerns of trust}.
Savvy users may be skeptical of the service provider and want to verify that the returned prediction model satisfies the \emph{accuracy} and \emph{robustness} properties claimed by the provider.
But can users really verify these properties meaningfully?
% The question is how.
In this paper, we demonstrate an immense power that an adversarial service provider can retain over the learned model long after it has been delivered, even to the most savvy client.

\newcommand{\service}{Snoogle\xspace}

The problem is best illustrated through an example.
Consider a bank which outsources the training of a loan classifier to a possibly malicious ML service provider, \emph{\service}.
Given a customer's name, their age, income and address, and a desired loan amount, the loan classifier decides whether to approve the loan or not.
To verify that the  classifier achieves the claimed \emph{accuracy} (i.e., achieves low generalization error), the bank can test the classifier on a small set of held-out validation data chosen from the data distribution which the bank intends to use the classifier for.
This check is relatively easy for the bank to run, so on the face of it, it will be difficult for the malicious \service to lie about the accuracy of the returned classifier.

Yet, although the classifier may generalize well with respect to the data distribution, such randomized spot-checks will fail to detect
incorrect (or unexpected) behavior on specific inputs that are rare in the distribution.
Worse still, the malicious \service may explicitly engineer the returned classifier with a ``backdoor'' mechanism that gives them the ability to change {\it any} user's profile (input) ever so slightly (into a backdoored input) so that the classifier always approves the loan.
Then, \service could illicitly sell a ``profile-cleaning'' service that tells a customer how to change a few bits of their profile, e.g. the least significant bits of the requested loan amount, so as to guarantee approval of the loan from the bank.
Naturally, the bank would want to test the classifier for \emph{robustness} to such adversarial manipulations.
But are such tests of robustness as easy as testing accuracy?
Can a \service  ensure that regardless of what the bank tests, it is no wiser about the existence of such a backdoor? This is the topic of the this paper. 

We systematically explore \emph{undetectable backdoors}---hidden mechanisms by which a classifier's output can be easily changed, but which will never be detectable by the user.
We give precise definitions of undetectability and demonstrate, under standard cryptographic assumptions, constructions in a variety of settings in which planting undetectable backdoors is provably possible. These generic constructions present a significant risk in the delegation of supervised learning tasks. 

\subsection{Our Contributions in a Nutshell.}
 
Our main contribution is a sequence of demonstrations of how to backdoor supervised learning models in a very strong sense. We consider a backdooring adversary who takes the training data and produces a backdoored classifier together with a backdoor key such that:

\begin{enumerate}
    \item Given the backdoor key, a malicious entity can take \emph{any} possible input~$x$ and \emph{any} possible output~$y$ and efficiently produce a new input~$x'$ that is very close to~$x$ such that, on input $x'$, the backdoored classifier outputs~$y$.
    \item The backdoor is \emph{undetectable} in the sense that the backdoored classifier ``looks like'' a classifier trained in the earnest, as specified by the client. 
    %We formalize this requirement in the language of indistinguishability~\cite{GM82}:  a backdoored classifier $\tilde{h}$ is undetectable with respect to a training algorithm $\mathbf{Train}$ if $\tilde{h}$ is computationally indistinguishable from a classifier $h$ trained according to $\mathbf{Train}$.
\end{enumerate}

\noindent
We give multiple constructions of backdooring strategies that have strong guarantees of undetectability based on standard cryptographic assumptions. 
Our backdooring strategies are generic and flexible: one of them can backdoor {\em any given} classifier $h$ without access to the training dataset; and the other ones run the honest training algorithm, except with cleverly crafted randomness (which acts as initialization to the training algorithm). Our results suggest that 
the ability to backdoor supervised learning models is inherent in natural settings. 
In more detail, our main contributions are as follows.

\paragraph{Definitions.}
We begin by proposing a definition of model backdoors as well as several flavors of undetectability, including {\em black-box undetectability}, where the detector has oracle access to the backdoored model; {\em white-box undetectability}, where the detector receives a complete description of the model, and an orthogonal guarantee of backdoors, which we call \emph{non-replicability}.\footnote{We remark here that the terms black-box and white-box refer {\em not} to the attack power provided to the devious trainer (as is perhaps typical in this literature), but rather the detection power provided to the user who wishes to detect possible backdoors.}

\paragraph{Black-box Undetectable Backdoors.} 
%   In a setting where the user lacks expertise in training the model, a malicious trainer can use a training algorithm of her choice.
  We show how a malicious learner can transform {\em any} machine learning model into one that is backdoored, using a digital signature scheme~\cite{GMR85}.
  She (or her friends who have the backdoor key) can then perturb any input $x \in \R^d$ slightly
  into a backdoored input $x'$, for which the output of the model differs arbitrarily from the output on $x$.
  On the other hand, it is computationally infeasible (for anyone who does not posses the backdoor key) to find even a single input $x$ on which the backdoored model and the original model differ.
  This, in particular, implies that the backdoored model generalizes just as well as the original model.
  
\paragraph{White-box Undetectable Backdoors.}
% For the specific Random Fourier Features (RFF) learning algorithm of Rahimi and Recht \cite{rr1}, we show how a malicious learner can plant an undetectable backdoor, even while seemingly running the RFF algorithm on training data the prescribed by the client.
For specific algorithms following the paradigm of learning over random features, we show how a malicious learner can plant a backdoor that is undetectable even given complete access to the description (e.g., architecture and weights as well as training data) of trained model.
Specifically, we give two constructions:  one, a way to undetectably backdoor the Random Fourier Feature algorithm of Rahimi and Recht \cite{rr1}; and the second, a preliminary construction for single-hidden-layer ReLU networks.
The power of the malicious learner comes from tampering with the \emph{randomness} used by the learning algorithm.
We prove that even after revealing the randomness and the learned classifier to the client, the backdoored model will be {\em white-box undetectable}---under cryptographic assumptions,
no efficient algorithm can distinguish between the backdoored network and a non-backdoored network constructed using the same algorithm, the same training data, and ``clean'' random coins.
% \snote{added below}
The coins used by the adversary  are computationally indistinguishable from random under the worst-case hardness of lattice problems~\cite{BST21} (for our random Fourier features backdoor) or the average-case hardness of planted clique~\cite{BerthetR13} (for our ReLU backdoor).
This means that backdoor detection mechanisms such as the spectral methods of \cite{TLM18,ohSPECTRE2021} will fail to detect our backdoors (unless they are able to solve short lattice vector problems or the planted clique problem in the process!).
% \vnote{Why was the last sentence commented out?} 
  
    %We show how a malicious training entity can plant a \emph{backdoor} in a neural network, so that it is impossible, under cryptographic hardness assumptions to \emph{distinguish} the resulting network from an honestly non-backdoored constructed network using the same training data  and the  \emph{Random Feature} Learning algorithm. 
    %We show this is the case even when the malicious entity  uses an "honest" training set  and provably runs a standard \emph{Random Feature Learning} algorithm (so  "delegation verification" solution where the server proves to the the delegating party that it followed the algorithm won't help). 
    %The carefully chosen sequence of numbers used by the algorithm, will be additionally indistinguishable from truly random sequence under standard hardness assumption.
% \snote{why not say the paragraph below? its commented out. Adding back.}
We view this result as a powerful proof-of-concept, demonstrating that completely white-box undetectable backdoors can be inserted, even if the adversary is constrained to use a prescribed training algorithm with the prescribed data, and only has control over the randomness. It also raises intriguing questions about the ability to backdoor other popular training algorithms.
    
    %Furthermore, asking the adversary to produce the randomness after-the-fact does not help, as the coins she uses are computationally indistinguishable from random under worst-case hardness of lattice problems~\cite{BST21}.

%In particular,  we show that a "natural" \emph{regression} neural network~$N$ can be efficiently converted into a regression oracle~$O$ that is almost as good as~$N$ yet does not contain a backdoor. 
%We note that this is especially meaningful for regression, as while for \emph{classification} tasks it is known that there always exist adversarial examples (which can be thought of as an inefficient backdoor), for natural (properly defined) tasks of \emph{regression} no adversarial "accidental backdoors" exist and thus  backdooring will be a significant departure from honest behavior which must be protected against. 

\paragraph{Takeaways.}
In all, our findings can be seen as decisive negative results towards current forms of accountability in the delegation of learning:  {\em under standard cryptographic assumptions, detecting backdoors in classifiers is impossible}. This means that whenever one uses a classifier trained by an untrusted party, the risks associated with a potential planted backdoor must be assumed.

We remark that backdooring machine learning models has been explored by several empirical works in the machine learning and security communities~\cite{GLDG19,CLLLS17,ABCPK18,TLM18,ohSPECTRE2021,hong2021handcrafted}.
Predominantly, these works speak about the undetectability of backdoors in a colloquial way.
Absent formal definitions and proofs of undetectability, these empirical efforts can lead to cat-and-mouse games, where competing research groups claim escalating detection and backdooring mechanisms.
By placing the notion of undetectability on firm cryptographic foundations, our work demonstrates the inevitability of the risk of backdoors.
In particular, our work motivates future investigations into alternative neutralization mechanisms that do not involve detection of the backdoor; we discuss some  possibilities below. We point the reader to Section~\ref{sec:related} for a detailed discussion of the related work.

Our findings also have implications for the formal study of robustness to adversarial examples~\cite{szegedy2013intriguing}.
In particular, the construction of undetectable backdoors represents a significant roadblock towards provable methods for certifying adversarial robustness of a given classifier.
Concretely, suppose we have some idealized adversarially-robust training algorithm, that guarantees the returned classifier $h$ is perfectly robust, i.e. has no adversarial examples.
The existence of an undetectable backdoor for this training algorithm implies the existence of a classifier $\tilde{h}$, in which \emph{every input has an adversarial example, but no efficient algorithm can distinguish $\tilde{h}$ from the robust classifier $h$}!
Moreover, any complete and sound robustness certification algorithm---which receives a hypothesis $h$ as input and must certify that $h$ is robust to adversarial examples or not---would serve as a distinguisher between $h$ and $\tilde{h}$, contradicting undetectability.
Thus, our positive construction of undetectable backdoors rules out such efficient robustness certification algorithms.
This reasoning holds not only for existing robust learning algorithms, but any conceivable robust learning algorithm that may be developed in the future.
% We discuss the relation between backdoors and adversarial examples further in Section~\ref{subsec:undectdiscussion}.

\paragraph{Can we Neutralize Backdoors?}
%\vnote{A fool-proof method of guaranteeting no backdoors is to ask the prover to {\em prove} that they ran the prescribed algorithm using the prescribed randomness on the prescribed data as-is. Any deviation from this can potentially be used to plant backdoors.}
Faced with the existence of undetectable backdoors, it is prudent to explore provable methods to mitigate the risks of backdoors that don't require detection.
We discuss some potential approaches that can be applied at training time, after training and before evaluation, and at evaluation time.
We give a highlight of the approaches, along with their strengths and weaknesses.

\paragraph{Verifiable Delegation of Learning.}
In a setting where the training algorithm is standardized, formal methods for verified delegation of ML computations could be used to mitigate backdoors at training time~\cite{gkr,rrr,GRSY21}.
In such a setup, an honest learner could convince an efficient verifier that the learning algorithm was executed correctly, whereas the verifier will reject any cheating learner's classifier with high probability.
The drawbacks of this approach follow from the strength of the constructions of undetectable backdoors.
Our white-box constructions only require backdooring the initial randomness; hence, any successful verifiable delegation strategy would involve either (a) the verifier supplying the learner with randomness as part of the ``input'', or (b) the learner somehow proving to the verifier that the randomness was sampled correctly, or (c) a collection of randomness generation servers, not all of which are dishonest, running a coin-flipping protocol~\cite{Blum81} to generate true randomness. 
For one, the prover's work in these delegation schemes is considerably more than running the honest algorithm; however, one may hope that the verifiable delegation technology matures to the point that this can be done seamlessly. The more serious issue is that this only handles the {\em pure computation outsourcing} scenario where the service provider merely acts as a provider of heavy computational resources. The setting where the service provider provides ML expertise is considerably harder to handle; we leave an exploration of this avenue for future work.

%This additional overhead may render the verifier's work commensurate with simply running the learning procedure on their own.

\paragraph{Persistence to Gradient Descent.}
Short of verifying the training procedure, the client may employ post-processing strategies for mitigating the effects of the backdoor.
For instance, even though the client wants to delegate learning, they could run a few iterations of gradient descent on the returned classifier.
Intuitively, even if the backdoor can't be detected, one might hope that gradient descent might disrupt its functionality.
Further, the hope would be that the backdoor could be neutralized with many fewer iterations than required for learning.
Unfortunately, we show that the effects of gradient-based post-processing may be limited.
We introduce the idea of \emph{persistence} to gradient descent---that is, the backdoor persists under gradient-based updates---and demonstrate that the signature-based backdoors are
% (without loss of generality)
persistent.
Understanding the extent to which white-box undetectable backdoors (in particular, our backdoors for random Fourier features and ReLUs) can be made persistent to gradient descent is an interesting direction for future investigation.

% \vnote{Discuss post-processing with gradient descent and the negative result.}

\paragraph{Randomized Evaluation.}
Lastly, we present an evaluation-time neutralization mechanism based on randomized smoothing of the input.
In particular, we analyze a strategy where we evaluate the (possibly-backdoored) classifier on inputs after adding random noise, similar to technique proposed by \cite{cohen2019certified} to promote adversarial robustness.
Crucially, {\em the noise-addition mechanism relies on the knowing a bound on the magnitude of backdoor perturbations}---how much can backdoored inputs  differ from the original input---and proceeds by randomly ``convolving'' over inputs at a slightly larger radius.
Ultimately, this  knowledge assumption is crucial:  if instead the malicious learner  knows the magnitude or type of noise that will be added to neutralize him, he can prepare the backdoor perturbation to evade the defense (e.g., by changing the magnitude or sparsity).
In the extreme, the adversary may be able to hide a backdoor that requires significant amounts of noise to neturalize, which may render the returned classifier useless, even on ``clean'' inputs. Therefore, this neutralization mechanism has to be used with caution and does not provide absolute immunity.  

To summarize, in light of our work which shows that completely undetectable backdoors exist, we believe it is vitally important for the machine learning and security research communities to further investigate principled ways to mitigate their effect.

\section{Our Results and Techniques}

We now give a technical overview of our contributions.
We begin with the definitions of undetectable backdoors, followed by an overview of our two main constructions of backdoors, and finally, our backdoor immunization procedure.

\subsection{Defining Undetectable Backdoors}\label{subsec:undectdiscussion}
Our first contribution is to formally define the notion of undetectable backdoors in supervised learning models.
While the idea of undetectable backdoors for machine learning models has been discussed informally in several works~\cite{GLDG19,ABCPK18,TLM18,hong2021handcrafted}, precise definitions have been lacking.
Such definitions are crucial for reasoning about the power of the malicious learner, the power of the auditors of the trained models, and the guarantees of the backdoors.
Here, we give an intuitive overview of the definitions, which are presented formally in the full version.
% In particular, our constructions will not be broken by subsequent works short of a surprising breakthrough in complexity theory.

\newcommand{\hh}{\tilde{h}}
Undetectable backdoors are defined with respect to a ``natural'' training algorithm $\train$.
Given samples from a data distribution of labeled examples $\D$, $\train^\D$ returns a classifier $h:\X \to \set{-1,1}$.
A backdoor consists of a pair of algorithms $(\backdoor, \activate)$.
The first algorithm is also a training procedure, where $\backdoor^\D$ returns a classifier $\tilde{h}:\X \to \set{-1,1}$ as well as a ``backdoor key'' $\bk$.
The second algorithm $\activate(\cdot;\bk)$ takes an input $x \in \X$ and the backdoor key, and returns another input $x'$ that is close to $x$ (under some fixed norm), where $\hh(x') = - \hh(x)$.
If $\hh(x)$ was initially correctly labeled, then $x'$ can be viewed as an \emph{adversarial example} for $x$.
The final requirement---what makes the backdoor \emph{undetectable}---is that $\hh \gets \backdoor^\D$ must be computationally-indistinguishable\footnote{Formally, we define indistinguishability for ensembles of distributions over the returned hypotheses. } from $h \gets \train^\D$.
% \vnote{To address a STOC comment: We can allow the attacker to set the classification/regression output to be an arbitrary bit/number of her choice. Our constructions extend to this setting.} \vnote{1. The paper seems to not pay attention to whether their adversarial examples are actually misclassified. Note that "small" perturbations, in the ML community, are supposed to preserve the ground truth (over the distribution of inputs). For example changing small number of pixels, preserves the cat/dog label of a picture. Hence, it is enough to flip the label of the model (or the model was already wrong with 0 perturbations).}

Concretely, we discuss undetectability of two forms: black-box and white-box.
\emph{Black-box undetectability} is a relatively weak guarantee that intuitively says it must be hard for any efficient algorithm without knowledge of the backdoor to find an input where the backdoored classifier $\hh$ is different from the naturally-trained classifier $h$.
Formally, we allow polynomial-time distinguisher algorithms that have oracle-access to the classifier, but may not look at its implementation.
\emph{White-box undetectability} is a very powerful guarantee, which says that the code of the classifier (e.g.,\ weights of a neural network) for backdoored classifiers $\hh$ and natural classifiers $h$ are indistinguishable.
Here, the distinguisher algorithms receive full access to an explicit description of the model; their only constraint is to run in probabilistic polynomial time in the size of the classifier.

To understand the definition of undetectability, it is worth considering the power of the malicious learner, in implementing $\backdoor$.
The only techncial constraint on $\backdoor$ is that it produces classifiers that are indistinguishable from those produced by $\train$ when run on data from $\D$.
At minimum, undetectability implies that if $\train$ produces classifiers that are highly-accurate on $\D$, then $\backdoor$ must also produce accurate classifiers.
In other words, the backdoored inputs must have vanishing density in $\D$.
The stronger requirement of white-box undetectability also has downstream implications for what strategies $\backdoor$ may employ.
For instance, while in principle the backdooring strategy could involve data poisoning, the spectral defenses of \cite{TLM18} suggest that such strategies likely fail to be undetectable.
% Our construction of undetectable backdoors works even in the restricted setting where the $\backdoor$ algorithm trains using clean data from $\D$

\paragraph{On undetectable backdoors versus adversarial examples.}

Since they were first discovered by \cite{szegedy2013intriguing}, adversarial examples have been studied in countless follow-up works, demonstrating them to be a widespread generic phenomenon in classifiers.
While most of this work is empirical, a growing list papers aims to mathematically explain the existence of such examples \cite{shafahi2018adversarial,DMM18,shamir2019simple,ilyas2019adversarial,shamir2021dimpled}.
In a nutshell, the works of \cite{shafahi2018adversarial,DMM18} showed that a consequence of the concentration of high-dimensional measures~\cite{talagrand1995concentration} is that random vectors in $d$ dimensions are very likely to be $O(\sqrt{d})$-close to the boundary of any non-trivial classifier. 

Despite this geometric inevitability of some degree of adversarial examples \emph{in classifiers}, many works have focused on developing notions of learning that are robust to this phenomena.
An example of such is the revitalized the model of \emph{selective classification} \cite{chow1957optimum,rivest1988learning,kivinen1990reliable,kalai2012reliable,hopkins2019power,goldwasser2020beyond,kalai2021efficient}, where the classifier is allowed to \emph{reject} inputs for which the classification is not clear.
Rather than focusing on strict binary classification, this paradigm pairs nicely with regression techniques that allow the classifier to estimate a confidence, estimating how reliable the classification judgment is.
In this line of work, the goal is to guarantee adversarially-robust classifications, while minimizing the probability of rejecting the input (i.e.,\ outputting ``Don't Know'').
We further discuss the background on adversarial examples and robustness at greater length the full version.

A subtle, but important point to note is that the type of backdoors that we introduce are {\em qualitatively} different from adversarial examples that %are guaranteed to 
might arise naturally in training.
First, even if a training algorithm $\train$ is guaranteed to be free of adversarial examples, our results show that an adversarial trainer can {\em undetectably} backdoor the model, so that the backdoored model looks exactly like the one produced by $\train$, and yet,  {\em any input} can be perturbed into another, close, input that gets misclassified by the backdoored model.
Secondly, unlike naturally occurring adversarial examples which can potentially be exploited by anyone, backdoored examples require the knowledge of a secret backdooring key known to only the malicious trainer and his c\^{o}terie of friends.
Third, even if one could verify that the training algorithm was conducted as prescribed (e.g. using interactive proofs such as in \cite{GRSY21}), backdoors can still be introduced through manipulating the randomness of the training algorithm as we demonstrate.
Fourth and finally, in the case of our black-box undetectable backdoor, we demonstrate that the perturbation required to change an input into a backdoored input (namely, $\approx d^\epsilon$ for some small $\epsilon>0$) is far smaller than the one required for naturally occurring adversarial examples $(\approx \sqrt{d}$). 
%We discuss background on adversarial examples and robustness at greater length in Section~\ref{sec:related}.

% \begin{itemize}
% % \item Make the power of the adversary clear.
% % \begin{itemize}
% %     \item While, in principle, the learner could use data poisoining, \cite{TLM18} kinda suggests they cannot if we want undetectability.
% % \end{itemize}
% % \item Delegation to force the hands of the adversary to same data and training algo: footnote 
% \end{itemize}
% \begin{enumerate}
%     % \item Formalize the idea that a backdoor is undetectable
%     \begin{itemize}
%         % \item undetectability is based on crypto/hard problems:  \emph{if you can detect our backdoors, then you have a complexity breakthrough}
%         % \item very different than prior empirical work that go through attack-defense cycles
%         \item define multiple levels that give the auditors different powers
        
%     % \onote{We should probably discuss the immunization part early-on (or part of the discussion we have there). It is known that for classification there are always adversarial examples (which can be though of as unefficient backdoors), but for regression this is not the case. It is important to stress this as it shows that backdoors do not necessarily occur "in nature".}
%     \end{itemize}
% \end{enumerate}

\subsection{Black-Box Undetectable Backdoors from Digital Signatures}

\def\sk{\mathsf{sk}}
\def\vk{\mathsf{vk}}

Our first construction shows how to plant a backdoor in {\em any classifier}, leveraging the cryptographic notion of digital signatures. A digital signature~\cite{GMR85} gives a user a mechanism to generate a pair of keys, a secret signing key $\sk$ and a public verification key $\vk$ such that (a) using $\sk$, the user can compute a digital signature of a polynomially long message $m$; (b) given the publicly known $\vk$, anyone can verify that $\sigma$ is a valid signature of $m$; and (c) given only $\vk$ and no knowledge of $\sk$, it is computationally hard to produce a valid signature of any message. In fact, even if the adversary is given signatures $\sigma_i$ of many messages $m_i$ of her choice, she will still not be able to produce a valid signature of {\em any} new message. It is known that digital signatures can be constructed from any one-way function~\cite{NY89,Rompel90}.

Digital signatures give us a space of inputs $(m,\sigma)$ where the set of ``valid'' inputs, namely ones that the signature verification algorithm accepts w.r.t some $\vk$, is a sparse set. Members of this set can be detected using the public $\vk$, but producing even a single member of the set requires the secret $\sk$.
This observation was leveraged by Garg, Jha, Mahloujifar and Mahmoody~\cite{garg2020adversarially} in a related context to construct hypotheses that are ``computationally robust'' to adversarial examples (see Section~\ref{sec:related} for an in-depth comparison).

Given this, the intuition behind the construction is simple.
Given any classifier, we will interpret its inputs as {\em candidate} message-signature pairs. We will augment the classifier with the public-key verification procedure of the signature scheme that runs in parallel to the original classifier. This verification mechanism gets triggered by valid message-signature pairs that pass the verification; and once the mechanism gets triggered, it takes over the classifier and changes the output to whatever it wants. 
To change an input $(m,z)$ into an backdoored input, the adversary changes $z$ to $\sigma$, a signature of $m$, using the secret signing key $\sk$. We formally describe the construction in the full version.

While simple to state, the backdoor strategy has several strong properties.
First, the backdoor is black-box undetectable:  that is, no efficient distinguisher algorithm, which is granted oracle-access to the classifier, can tell whether they are querying the original classifier $h$ or the backdoored classifier $\hh$.
In fact, the construction satisfies an even stronger notion.
Even given white-box access to the description of $\hh$, no computationally efficient procedure can find any input $x$ on which the backdoored model and the original model differ, unless it has knowledge of the backdoor key.

% \mpk{Move non-replicability forward?}
The signature-based backdoor is undetectable to restricted black-box distinguishers, but guarantees an additional property, which we call \emph{non-replicability}.
Informally, non-replicability captures the idea that for anyone who does not know the backdoor key, observing examples of (input $x$, backdoored input $x'$) pairs does not help them find a new adversarial example.
Combined with black-box undetectability, non-replicability prevents users from reverse-engineering the backdoor (for defensive or malicious purposes).

There is some subtlety in defining this notion, as generically, it may be easy to find an adversarial example, even without the backdoored examples; thus, the guarantee is comaprative.
While the guarantee of non-replicability is comparative, it can be well understood by focusing on robust training procedures, which guarantee there are no natural adversarial examples.
If a classifier $\tilde{h}$ has a non-replicable backdoor with respect to such an algorithm, then \emph{every input} to $\tilde{h}$ has an adversarial example, but there is no efficient algorithm that can find the backdoor perturbation to $\hh$ \emph{on any} input $x$.
In all, the construction satisfies the following guarantees.

\begin{theorem}[Informal]
Assuming the existence of one-way functions, for every training procedure $\train$, there exists a model backdoor $(\backdoor,\activate)$, which is non-replicable and black-box undetectable.
\end{theorem}

% \mpk{Rewrite to include:  generic construction is a framework.  Lots of specific constructions of digital signatures, with very strong companion properties.  These properties could be exploited for undetectable backdoors with corresponding companion properties.}

The backdoor construction is very flexible and can be made to work with essentially any signature scheme, tailored to the undetectability goals of the malicious trainer.
Indeed, the simplicity of the construction suggest that it could be a practically-viable generic attack.
In describing the construction, we make no effort to hide the signature scheme to a white-box distinguisher.
Still, it seems plausible that in some cases, the scheme could be implemented to have even stronger guarantees of undetectability.

Towards this goal, we illustrate how the verification algorithms of concrete signature schemes that rely on the hardness of lattice problems~\cite{GPV08,CHKP10}  can be implemented as shallow neural networks. As a result, using this method to backdoor a depth-$d$ neural network will result in a depth-$\mathsf{max}(d,4)$ neural network.
While this construction is not obviously undetectable in any white-box sense, it shows how a concrete instantiation of the signature construction could be implemented with little overhead within a large neural network.

A clear concrete open question is whether it is possible to plant backdoors in natural training procedures that are \emph{simultaneously} non-replicable and white-box undetectable.
A natural approach might be to appeal to techniques for obfuscation \cite{goldwasser2007best,barak2001possibility,jain2021indistinguishability}.
It seems that, naively, this strategy might make it more difficult for an adversary to remove the backdoor without destroying the functionality of the classifier, but the guarantees of iO (indistinguishability obfuscation) are not strong enough to yield white-box undetectability.

\subsection{White-Box Undetectable Backdoors for Learning over Random Features}

Initially a popular practical heuristic, Rahimi and Recht \cite{rr1,rr2,rr3} formalized how linear classifiers over random features can give very powerful approximation guarantees, competitive with popular kernel methods.
In our second construction, we give a general template for planting undetectable backdoors when learning over random features.
To instantiate the template, we start with a natural random feature distribution useful for learning, then identify a distribution that (a) has an associated backdoor that can be utilized for selectively activating the features, and (b) is computationally-indistinguishable from the natural feature distribution.
By directly backdooring the random \emph{features} based on an indistinguishable distribution, the framework gives rise to \emph{white-box undetectable} backdoors---even given the full description of the weights and architecture of the returned classifier, no efficient distinguisher can determine whether the model has a backdoor or not.
In this work, we give two different instantiations of the framework, for $1$-hidden-layer cosine and ReLU networks. Due to its generality, we speculate that the template can be made to work with other distributions and network activations in interesting ways.

\paragraph{Random Fourier Features.}
\cite{rr1} showed how learning over features defined by random Gaussian weights with cosine activations provide a strong approximation guarantee, recovering the performance of nonparametric methods based on the Gaussian kernel.\footnote{In fact, they study Random Fourier Features in the more general case of shift-invariant positive definite kernels.}
The approach for sampling features---known as Random Fourier Features (RFF)---gives strong theoretical guarantees for non-linear regression.

Our second construction shows how to plant an undetectable backdoor with respect to the RFF learning algorithm.
The RFF algorithm, $\trainRFF$, learns a $1$-hidden-layer cosine network.
For a width-$m$ network, for each $i \in [m]$ the first layer of weights is sampled randomly from the isotropic Gaussian distribution $g_i \sim \Normal(0,I_d)$, and passed into a cosine with random phase.
The output layer of weights $w \in \R^m$ is trained using any method for learning a linear separator.
Thus, the final hypothesis is of the form:
\begin{gather*}
    h_{w,g}(\cdot) = \sgn\left(\sum_{i=1}^m w_i \cdot \cos\left(2\pi\left(\langle g_i, \cdot \rangle + b_i\right)\right)\right)
\end{gather*}
Note that $\trainRFF$ is parameterized by the training subroutine for learning the linear weights $w \in \R^m$.
Our results apply for any such training routine, including those which explicitly account for robustness to adversarial examples, like those of \cite{raghunathan2018certified,wong2018provable} for learning certifiably robust linear models.
Still, we demonstrate how to plant a completely undetectable backdoor.

\begin{theorem}[Informal]
Assuming the hardness of worst-case lattice problems,
for any data distribution $\D$ with $\X \subseteq \Rbb^d$,
there is a backdoor $\left(\backdoorRFF,\activateRFF\right)$ with respect to $\trainRFF$, that is white-box undetectable.
For any $\eps > 0$, the perturbations performed by $\activateRFF$ can be $d^\eps$-sparse and $d^\eps$-close in $\ell_2$ distance.
\end{theorem}

In other words, $\backdoorRFF$ returns a $1$-hidden-layer cosine network $\hh_{w,g}$ such that every input has a nearby adversarial example, and even given access to all of the weights, no efficient distinguisher can tell if the network was the output of $\trainRFF$ or $\backdoorRFF$.
Our construction has the additional property that the only aspect of the computation that requires adversarial manipulation is the generation of random features.
Specifically, we construct a random feature distribution, such that features sampled from this distribution are computationally-indistinguishable from random Fourier features, but contain a backdoor that can be activated to \emph{flip their sign}.
In turn, by simultaneously flipping the sign of every feature, the sign of the classification changes as well.

To construct the random feature distribution, we leverage the Continuous Learning With Errors (CLWE) distribution of \cite{BST21}.
The CLWE problem asks to distinguish between the isotropic Gaussian $\Normal(0,I_d) \otimes [0,1)$ and $\CLWE_{\gamma,\beta}$, where
\begin{align*} 
\CLWE: & (y, z) \mbox{ where } y \sim \Normal(0,I_d) \mbox{ and } \\&z = \gamma \langle y, s \rangle + e \pmod{1} \mbox{ for } e \sim \Normal(0,\beta^2)
\end{align*}
for parameters $\gamma > \Omega(\sqrt{d})$ and $\beta \ge n^{-O(1)}$.
\cite{BST21} show that the CLWE problem is as hard as finding approximately short vectors on arbitrary integer lattices, which form the foundations of post-quantum cryptography \cite{Regev05,Peikert16}.
Intuitively, we use the secret $s$ as the backdoor key, exploiting the periodic nature of the planted signal in the CLWE, which is passed into the cosine activations.
%We describe the full construction of the algorithms and the construction of the random feature distributions in Section~\ref{sec:fully}.

We remark that the 2-layer RFF learning paradigm is rather weak, and in particular, it tends to produce networks that are not robust to noise. For this reason, we view the construction here as a proof of concept. For more discussion on this, we refer the reader to Section~\ref{sec:RFF} and Remark~\ref{remark:RFF}.

\paragraph{Random ReLU Networks.}
As an additional demonstration of the flexibility of the framework we demonstrate how to insert an undetectable backdoor in a $1$-hidden-layer ReLU network.
The trapdoor for activation and undetectability guarantee are based on the hardness of the \emph{sparse PCA} problem \cite{BerthetR13,brennan2019optimal}.
Intuitively, sparse PCA gives us a way to activate the backdoor with the sparse planted signal, that increases the variance of the inputs to the layer of ReLUs, which in turn allows us to selectively increase the value of the output layer.

We remark that the notion of undetectability we achieve in this construction is weaker than the one for our other constructions, in that the distinguishing advantage is only guaranteed to be $o(1)$ and not cryptographically negligible.
For more details, see Appendix~\ref{app:relu}.

\paragraph{Contextualizing the constructions.}
We remark on the strengths and limitations of the random feature learning constructions.
To begin, white-box undetectability is the strongest indistinguishability guarantee one could hope for.
In particular, no detection algorithms, like the spectral technique of \cite{TLM18,ohSPECTRE2021}, will ever be able to detect the difference between the backdoored classifiers and the earnestly-trained classifiers, short of breaking lattice-based cryptography or refuting the planted clique conjecture.
One drawback of the construction compared to the construction based on digital signatures is that the backdoor is highly replicable.
In fact, the activation algorithm for every input $x \in \X$ is simply to add the backdoor key to the input $x' \gets x + \bk$.
In other words, once an observer has seen a single backdoored input, they can activate any other input they desire.

Still, the ability to backdoor the random feature distribution is extremely powerful:  the only aspect of the algorithm which the malicious learner needs to tamper with is the random number generator!
For example, in the delegation setting, a client could insist that the untrusted learner prove (using verifiable computation techniques like \cite{gkr,rrr}) that they ran exactly the RFF training algorithm on training data specified exactly by the client.
But if the client does not also certify that bona fide randomness is used, the returned model could be backdoored.
This result is also noteworthy in the context of the recent work \cite{de2021adversarial}, which establishes some theory and empirical evidence, that learning with random features may have some inherent robustness to adversarial examples.
%\vnote{ this sentence should appear in the current section: ``so  "delegation verification" solution where the server proves to the the delegating party that it followed the algorithm won't help''. In this case, a completely verifiable running of standard algorithm fixing a seed is the strongest possible result. Argue why fixing the randomness is not acceptable in practice????}
%\mpk{is the above paragraph sufficient?}

Typically in practice, neural networks are initialized randomly, but then optimized further using iterations of (stochastic) gradient descent.
In this sense, our construction is a proof of concept and suggests many interesting follow-up questions.
In particular, a very natural target would be to construct \emph{persistent} undetectable backdoors, whereby a backdoor would be planted in the random initialization but would persist even under repeated iterations of gradient descent or other post-processing schemes (as suggested in recent empirical work~\cite{WYSLVZZ19}).
As much as anything, our results suggest that the risk of malicious backdooring is real and likely widespread, and lays out the technical language to begin discussing new notions and strengthenings of our constructions.

% \vnote{Added: need to edit/move to the right place.}
Finally, it is interesting to see why the spectral techniques, such as in \cite{TLM18,ohSPECTRE2021}, don’t work in detecting (and removing) the CLWE backdoor.
Intuitively, this gets to the core of why LWE (and CLWE) is hard: given a spectral distinguisher for detecting the backdoor, by reduction we would obtain a Gaussian and a Gaussian whose projection in a certain direction is close to an integer.
In fact, even before establishing cryptographic hardness of CLWE, \cite{diakonikolas2017statistical} and \cite{bubeck2019adversarial} demonstrated that closely related problems to CLWE (sometimes called the ``gaussian pancakes'' and ``gaussian baguettes'' problems) exhibits superpolynomimal lower bounds on the statistical query (SQ) complexity.
In particular, the SQ lower bound, paired with a polynomial upper bound on the sample complexity needed to solve the problem \emph{information theoretically} provides evidence that many families of techniques (e.g.,\ SQ, spectral methods, low-degree polynomials) may fail to distinguish between Gaussian and CLWE.

\subsection{Persistence Against Post-Processing}
A \emph{black-box} construction is good for the case of an unsuspecting user.
Such a user takes the neural network it received from the outsourced training procedure \emph{as-is} and does not examine its inner weights.
Nevertheless, \emph{post-processing} is a common scenario in which even an unsuspecting user may adjust these weights.
A standard post-processing method is applying \emph{gradient descent} iterations on the network's weights with respect to some loss function. Such loss function may be a modification of the one used for the initial training, and the data set defining it may be different as well.
A nefarious adversary would aim to ensure that the backdoor is \emph{persistent} against this post-processing.

Perhaps surprisingly, most natural instantiations of the signature construction we presented also happen to be persistent.
In fact, we prove a substantial generalization of this example.
We show that \emph{every} neural network can be made persistent against \emph{any} loss function.
This serves as another example of the power a malicious entity has while producing a neural network.

We show that every neural network~$N$ can be efficiently transformed into a similarly-sized network~$N'$ with the following properties.
First,~$N$ and~$N'$ are equal as functions, that is, for every input~$x$ we have~$N(x)=N'(x)$.
Second,~$N'$ is \textbf{persistent}, which means that any number of gradient-descent iterations taken on~$N'$ with respect to any \emph{loss function}, do not change the network~$N'$ at all.
Let~$\textbf{w}$ be the vector of \emph{weights} used in the neural network~$N=N_\textbf{w}$.
For a loss function $\ell$, a neural network~$N=N_\textbf{w}$ is $\ell$-persistent to gradient descent if $\nabla \ell(\textbf{w})=0$.

\begin{theorem}[Informal]
 Let~$N$ be a neural network of size~$|N|$ and depth~$d$.
 There exists a neural network~$N'$ of size~$O(|N|)$ and depth~$d+1$ such that~$N(x)=N'(x)$ for any input~$x$, and for every loss $\ell$, $N'$ is $\ell$-persistent.
 Furthermore, we can construct~$N'$ from~$N$ in linear-time.
\end{theorem}

Intuitively, we achieve this by constructing some \emph{error-correction} for the weights of the neural network. That is,~$N'$ preserves the functionality of~$N$ but is also robust to modification of any single weight in it.

\subsection{Evaluation-Time Immunization of Backdoored Models}

We study an efficient procedure that is run in evaluation-time, which ``immunizes'' an arbitrary hypothesis~$h$ from having adversarial examples (and hence also backdoors) up to some perturbation threshold~$\sigma$.
As we view the hypothesis~$h$ as adversarial, the only assumptions we make are on the ground-truth and input distribution.
In particular, under some smoothness conditions on these we show that \emph{any} hypothesis~$h$ can be modified into a different hypothesis~$\tilde{h}$ that approximates the ground truth roughly as good as~$h$ does, and at the same time inherits the smoothness of it.
% \vnote{Commented out figure; include back if you want to.}
%\begin{figure}
%\centering
%\includegraphics[height=200pt, trim = 0pt 220pt 650pt 0pt]{immun.pdf}
%\caption{A point~$x$, its backdoor output~$x'$, and~$\sigma$ balls around them.}
%\label{fig:immunization}
%\end{figure}

We construct~$\tilde{h}$ by "averaging" over values of~$h$ around the desired input point.
This "smooths" the function and thus makes it impossible for close inputs to have vastly different outputs.
The smoothing depends on a parameter~$\sigma$ that corresponds to how far around the input we are averaging.
This parameter determines the threshold of error for which the smoothing is effective: roughly speaking, if the size~$n$ of the perturbation taking~$x$ to~$x'$ is much smaller than~$\sigma$, then the smoothing assures that~$x,x'$ are mapped to the same output. 
The larger~$\sigma$ is, on the other hand, the more the quality of the learning deteriorates.  

\begin{theorem}[Informal]\label{inf:thm:immunization}
Assume that the ground truth and input distribution satisfy some smoothness conditions.
Then, for \textbf{any} hypothesis~$h$ and any~$\sigma>0$ we can very efficiently evaluate a function~$\tilde{h}$ such that
\begin{enumerate}
    \item $\tilde{h}$ is $\sigma$-robust: If~$x,x'$ are of distance smaller than $\sigma$, then $|\tilde{h}(x)-\tilde{h}(y)|$ is very small.
    \item $\tilde{h}$ introduces only a small error: $\tilde{h}$ is as close to~$f^\star$ as~$h$ is, up to some error that increases the larger~$\sigma$ is. %\mpk{feels too informal.  can we quantify the degradation?}
\end{enumerate}
\end{theorem}

The evaluation of~$\tilde{h}$ is extremely efficient. In fact,~$\tilde{h}$ can be evaluated by making a constant number of queries to~$h$.
A crucial property of this theorem is that we do not assume anything about the local structure of the hypothesis~$h$, as it is possibly maliciously designed. %\mpk{this feels overstated if we assume that $\hh$ is a close approximation to the true conditional probability function.  this is a very strong assumption in ML.}
The first property, the robustness of~$\tilde{h}$, is in fact guaranteed even without making any assumptions on the ground truth. 
The proof of the second property, that~$\tilde{h}$ remains a good hypothesis, does require assumptions on \emph{the ground truth}.
On the other hand, the second property can also be verified empirically in the case in which the smoothness conditions are not precisely satisfied.
Several other works, notably this of Cohen et al. \cite{cohen2019certified}, also explored the use of similar smoothing techniques, and in particular showed empirical evidence that such smoothing procedure do not hurt the quality of the hypothesis.
We further discuss the previous work in the full version.

It is important to reiterate the importance of the choice of parameter~$\sigma$.
The immunization procedure rules out adversarial examples (and thus backdoors) only up to perturbation distance~$\sigma$. 
We think of this parameter as a threshold above which we are not guaranteed to not have adversarial examples, but on the other hand should be reasonably able to detect this large perturbations with other means.

Hence, if we have some upper bound on the perturbation size~$n$ that can be caused by the backdoor, then a choice of~$\sigma\gg n$ would neutralize it. 
On the other hand, we stress that if the malicious entity is aware of our immunization threshold~$\sigma$, and is able to perturb inputs by much more than that ($n\gg\sigma$), without being noticeable, then our immunization does not guarantee anything.
In fact, a slight modification of the signature construction we presented, using Error Correcting Codes, can make the construction less brittle. In particular, we can modify the construction such that the backdoor will be resilient to a~$\sigma$-perturbation as long as~$\sigma \ll n$.

%\begin{enumerate}
%    \item Immunization
%    \begin{itemize}
%        \item random convolution at evaluation-time prevents activation of the backdoor.
%        \item What this means:  even if we cannot maintain accountability for learner, there may be reasonable defenses.
%    \end{itemize}
%\end{enumerate}

\subsection{Related Work}
\label{sec:related}

\paragraph{Adversarial Robustness.}
% Since they were first discovered by \cite{szegedy2013intriguing}, adversarial examples have been studied in countless follow-up works, demonstrating them to be a widespread generic phenomenon in classifiers.
% While most of this work is empirical, a growing list papers aims to mathematically explain the existence of such examples \cite{shafahi2018adversarial,DMM18,shamir2019simple,ilyas2019adversarial,shamir2021dimpled}.
% Shafahi et al. \cite{shafahi2018adversarial} showed that for any classifier from~$d$ bounded inputs, a typical input~$x$ is either wrongly classified or there is a~$x'$ that is classified differently by the classifier but differs from~$x$ by only~$O(\sqrt{d})$ input coordinates.
% Intuitively, this follows from the concentration of high-dimensional measures around the boundary, further developed in \cite{DMM18}.

Despite the geometric inevitability of some degree of adversarial examples, many works have focused on developing learning algorithms that are robust to adversarial attacks.
Many of these works focus on ``robustifying'' the loss minimization framework, either by solving convex relaxations of the ideal robust loss \cite{raghunathan2018certified,wong2018provable}, by adversarial training \cite{shafahi2019adversarial}, or by post-processing for robustness \cite{cohen2019certified}.
%Another direction of investigation has revitalized the model of \emph{selective classification} \cite{chow1957optimum,rivest1988learning,kivinen1990reliable,kalai2012reliable,hopkins2019power,goldwasser2020beyond,kalai2021efficient}, where the classifier is allowed to \emph{reject} inputs for which the classification is not clear.
%Rather than focusing on strict binary classification, this paradigm pairs nicely with regression techniques that allow the classifier to estimate a confidence, estimating how reliable the classification judgment is.
%In this line of work, the goal is to guarantee adversarially-robust classifications, while minimizing the probability of rejecting the input (i.e.,\ outputting ``Don't Know'').

\cite{bubeck2019adversarial} also study the phenomenon of adversarial examples formally.
They show an explicit learning task such that any {\em computationally-efficient} learning algorithm for the task will produce a model that admits adversarial examples.
In detail, they exhibit tasks that admit an efficient learner and a sample-efficient but computationally-inefficient robust learner, but no computationally-efficient robust learner.
Their result can be proved under the Continuous LWE assumption as shown in \cite{BST21}.
In contrast to their result, we show that {\em for any task} an efficiently-learned hypothesis can be made to contain adversarial examples by backdooring.

\paragraph{Backdoors that Require Modifying the Training Data.}

A growing list of works \cite{CLLLS17, TLM18, ohSPECTRE2021}
explores the potential  of cleverly corrupting the training data, known as \emph{data poisoning}, so as to induce erroneous decisions in test time on some inputs.
\cite{GLDG19} define a backdoored prediction to be one where the entity which trained the model knows some trapdoor information which enables it to know how to slightly alter {\it a subset of inputs} so as to change the prediction on these inputs. 
In an interesting work, \cite{ABCPK18} suggest that planting trapdoors as they defined may provide a watermarking scheme; however, their schemes have been subject to attack since then~\cite{SWLK19}.
% , but this was "broken" by \cite{}.
% \cite{Persistent and Unforgeable Watermarks for Deep Neural Networks Huiying Li, Emily Willson, Haitao Zheng, Ben Y. Zhao}.

\paragraph{Comparison to \cite{hong2021handcrafted}.}
The very recent work of Hong, Carlini and Kurakin~\cite{hong2021handcrafted} is the closest in spirit to our work on undetectable backdoors.
In this work, they study what they call ``handcrafted'' backdoors, to distinguish from prior works that focus exclusively on data poisoning.
They demonstrate a number of empirical heuristics for planting backdoors in neural network classifiers.
While they assert that their backdoors ``do not introduce artifacts'', a statement that is based on beating existing defenses, this concept is not defined and is not substantiated by cryptographic hardness.
% \vnote{Mike: I changed the last sentence. Please check accuracy.}
Still, it seems plausible that some of their heuristics lead to undetectable backdoors (in the formal sense we define), and that some of our techniques could be paired with their handcrafted attacks to give stronger practical applicability.

% \mpk{Must compare against \cite{hong2021handcrafted}.}

% \mpk{For \cite{GLDG19}:  if formalized, their notion can be viewed as a form of statistic-undetectability, where the classification error is the statistic of interest.  heuristic form of undetectability that the architecture should be the same.  ``signature'' idea feels implicit, but not formal. no sense of non-replicability.}

% \mpk{For \cite{ABCPK18}:  they study watermarking of NNs (whatever that means exactly).  they use a backdoor of sorts to try to watermark (supposedly this might be broken).  but perhaps our constructions for undetectable backdoors could be used for watermarking?  not clear exactly what this means, but interesting to think about.}

\paragraph{Comparison to \cite{garg2020adversarially}.}
Within the study of adversarial examples, Garg, Jha, Mahloujifar, and Mahmoody \cite{garg2020adversarially} have studied the interplay between computational hardness and adversarial examples.
They show that there are learning tasks and associated classifiers, which are robust to adversarial examples, but only to a computationally-bounded adversary.
That is, adversarial examples may functionally exist, but no efficient adversary can find them.
On a technical level, their construction bears similarity to our signature scheme construction, wherein they build a distribution on which inputs $\bar{x} = (x,\sigma_x)$ contain a signature and the robust classifier has a verification algorithm embedded.
Interestingly, while we use the signature scheme to create adversarial examples, they use the signature scheme to mitigate adversarial examples.
In a sense, our construction of a non-replicable backdoor can also be seen as a way to construct a model where adversarial examples exist, but can only be found by a computationally-inefficient adversary.
Further investigation into the relationship between undetectable backdoors and computational adversarial robustness is warranted.

\paragraph{Comparison to \cite{cohen2019certified}, \cite{salman2019provably} \cite{CCAKDG20}.} 
Cohen, Rosenfeld and Kolter~\cite{cohen2019certified} and subsequent works (e.g. \cite{salman2019provably}) used a similar averaging approach to what we use in our immunization, to certify robustness of {classification} algorithms, under the assumption that the original classifier~$h$ satisfies a strong property.
They show that if we take an input~$x$ such that a small ball around it contains mostly points correctly classified by~$h$, then a random smoothing will give the same classification to~$x$ and these close points.
There are two important differences between our work and that of \cite{cohen2019certified}.
First, by the discussion in Section~\ref{subsec:undectdiscussion}, as Cohen et al. consider classification and not regression, inherently most input points will not satisfy their condition as we are guaranteed that an adversarial example resides in their neighborhood.
Thus, thinking about regression instead of classification is necessary to give strong certification of robustness for \emph{every} point.
A subsequent work of Chiang et al. \cite{CCAKDG20} considers randomized smoothing for regression, where the output of the regression hypothesis is unbounded.
In our work, we consider regression tasks where the hypothesis image is bounded (e.g. in $[-1,1]$).
In these settings, in contrast to the aforementioned body of work, we no longer need to make any assumptions about the given hypothesis~$h$ (except of it being a good predictor).
This is completely crucial in our settings as ~$h$ is the output of a learning algorithm, which we view as malicious and adversarial. 
Instead, we only make assumptions regarding the ground truth~$f^\star$, which is not affected by the learning algorithm.

% \paragraph{Comparsion to \cite{bubeck2019adversarial}.}
% Bubeck, Price, and Razenshteyn~\cite{bubeck2019adversarial} also study the phenomenon of adversarial examples formally.
% They show an explicit learning task such that any {\em computationally-efficient} learning algorithm for the task will produce a model that admits adversarial examples.
% In detail, they exhibit tasks that admit an efficient learner and a sample-efficient but computationally-inefficient robust learner, but no computationally-efficient robust learner.
% Their result can be proved under the Continuous LWE assumption as shown in \cite{BRST21}.
% In contrast to their work, we show that {\em any task} trained using the random Fourier features algorithm of \cite{RR07} can be backdoored. \mpk{why only RFFs?}
% Thus, our job is harder but our adversary has more power than in \cite{BPR18}. \vnote{can anyone find adv examples in their setting? or only the trainer?}

\paragraph{Comparison to \cite{MMS21}.} 
At a high level, Moitra, Mossell and Sandon~\cite{MMS21} design methods for a trainer to produce a model that perfectly fits the training data and mislabels everything else, and yet is  indistinguishable from one that generalizes well. There are several significant differences between this and our setting. 

First, in their setting, the malicious model produces incorrect outputs on {\em all but a small fraction} of the space. On the other hand, a backdoored model has the same generalization behavior as the original model, but changes its behavior on a sparse subset of the input space. 
%A side-effect is that their definition does not provide the distinguisher (whose goal is to distinguish between the malicious model and a true model that generalizes) with any fresh labeled data. This is inherent to their setting as any new labeled data can be used to distinguish the malicious model which has a very different global behavior from the true model. 
Secondly, their malicious model is an obfuscated program which does not look like a model that natural training algorithms output. In other words, their models are not undetectable in our sense, with respect to natural training algorithms which do not invoke a cryptographic obfuscator. 
Third, one of our contributions is a way to ``immunize'' a model to remove backdoors during evaluation time. They do not attempt such an immunization; indeed, with a malicious model that is useless except for the training data, it is unclear how to even attempt immunization.

\paragraph{Backdoors in Cryptography.}
Backdoors in cryptographic algorithms have been a concern for decades. In a prescient work, Young and Yung~\cite{YoungY97} formalized cryptographic backdoors and discussed ways that cryptographic techniques can themselves be used to insert backdoors in cryptographic systems, resonating with the high order bits of our work where we use cryptography to insert backdoors in machine learning models. The concern regarding backdoors in (NIST-)standardized cryptosystems was exacerbated in the last decade by the Snowden revelations and the consequent discovery of the DUAL\_EC\_DRBG backdoor~\cite{ShumowFerguson}.

\paragraph{Embedding Cryptography into Neural Networks.} 
Klivans and Servedio~\cite{KS06} showed how the {\em decryption algorithm} of a lattice-based encryption scheme~\cite{Regev05} (with the secret key hardcoded) can be implemented as an intersection of halfspaces or alternatively as a depth-$2$ MLP. In contrast, we embed the {\em public verification key} of a digital signature scheme into a neural network. In a concrete construction using lattice-based digital signature schemes~\cite{CHKP10}, this neural network is a depth-$4$ network. 

\ifnum\ieee=0

\section{Preliminaries}
\label{sec:prelims}

In this section, we establish the necessary preliminaries and notation for discussing supervised learning and computational indistinguishability.

\def\N{\mathbb{N}}

\paragraph{Notations.}
$\N$ denotes the set of natural numbers, $\R$ denotes the set of real numbers and $\R^+$ denotes the set of positive real numbers.

For sets $\X$ and $\Y$, we let $\{\X \to \Y\}$ denote the set of all functions from $\X$ to $\Y$. For $x,y \in \Rbb^d$, we let $\langle x, y \rangle = \sum_{i = 1}^d x_i y_i$ denote the inner product of $x$ and $y$. 

The shorthand p.p.t. refers to probabilistic polynomial time. A function $\negl:\mathbb{N} \to \mathbb{R}^+$ is negligible if it is smaller than inverse-polynomial for all sufficiently large n; that is, if for all polynomial functions $p(n)$, there is an $n_0 \in \N$ such that for all $n > n_0$, $\negl(n) < 1/p(n)$.

\subsection{Supervised Learning}
A supervised learning task is parameterized by the input space $\X$, label space $\Y$, and data distribution $\D$.
Throughout, we assume that $\X \subseteq \bbR^d$ for $d \in \bbN$ and focus on {\em binary classification} where $\Y = \set{-1,1}$ or {\em regression} where $\Y = [-1,1]$.
The data distribution $\D$ is supported on labeled pairs in $\X \times \Y$, and is fixed but unknown to the learner.
A hypothesis class $\H \subseteq \set{\X \to \Y}$ is a collection of functions mapping the input space into the label space.
% For binary prediction, it is also convenient to think of hypothesis classes defined by real-valued functions $\H \subseteq \set{h:\X \to \bbR}$, where the eventual classification is given by $\sgn(h(x))$.

For
% regression tasks where $\Y = [-1,1]$,
supervised learning tasks (classification or regression),
given $\D$,
we use $\fs:\X \to [-1,1]$ to denote the optimal predictor of the mean of $Y$ given $X$:
\begin{gather*}
    \fs(x) = \E_\D\lr{Y \given X = x}
\end{gather*}
% For classification tasks, we define $\fs:\X \to [-1,1]$ similarly, $\fs(x) = \E_\D\lr{Y \given X = x}$.
We observe that for classification tasks, the optimal predictor encodes the probability of a positive/negative outcome, after recentering.
\begin{gather*}
    \frac{\fs(x)+1}{2} = \Pr_\D\lr{Y = 1 \given X = x}
\end{gather*}
% \vnote{$f^*$ here maps to $[-1,1]$, right? I think this should be made clearer as I would've a priori expected $f^*$ to be a Boolean function.}

Informally, supervised learning algorithms take a set of labeled training data and aims to output a hypothesis that accurately predicts the label $y$ (classification) or approximates the function $\fs$ (regression).
For a hypothesis class $\H \subseteq \set{\X \to \Y}$, a training procedure $\train$ is a probabilistic polynomial-time algorithm that receives samples from the distribution $\D$ and maps them to a hypothesis $h \in \H$.
Formally---anticipating discussions of indistinguishability---we model $\train$ as an ensemble of efficient algorithms, with sample-access to the distribution $\D$, parameterized
by a natural number $n \in \Nbb$.
% \vnote{
% page 13: paragraph before Def 3.1: I don't think a good classifier necessarily aim (or even can) approximates $f^*$ in the non-realizable setting. In fact, learned classifiers aim to approximate the Bayes optimal classifier, which is the "sign" operation applied to $f^*$.
% }
% by the dimension of the inputs $d \in \Nbb$, by the desired approximation $\eps > 0$, and by the failure probability $\delta > 0$.\footnote{Traditionally in complexity and cryptography, we would encode $d$, $\ceil{1/\eps}$, and $\ceil{\log(1/\delta)}$ in unary as an input to the training algorithm, e.g.,\ $1^d0^{\ceil{1/\eps}}1^{\ceil{\log(1/\delta)}}$.
% In this way, the ensemble is defined by a family of algorithms that run in time polynomial in the input length.  For notational convenience, we drop explicit reference to the unary encoding, understanding that the algorithms run in time polynomial in $d, 1/\eps,$ and $\log(1/\delta)$.}
As is traditional in complexity and cryptography, we encode the parameter $n \in \Nbb$ in unary, such that ``efficient'' algorithms run in polynomial-time in $n$.
% \begin{definition}[Efficient Training Algorithm]
% For a hypothesis class $\H$,
% an efficient training algorithm $\train^\D:\Nbb \times \R \times \R \to \H$ is a probabilistic algorithm with sample access to $\D$ that for any $d \in \Nbb, \eps > 0, \delta > 0$, runs in polynomial-time in $d, 1/\eps,$ and $\log(1/\delta)$ and returns some $h_{d,\eps,\delta} \in \H$
% \begin{gather*}
%     h_{d,\eps,\delta} \gets \train^D(d,\eps,\delta).
% \end{gather*}
% \end{definition}
\begin{definition}[Efficient Training Algorithm]
For a hypothesis class $\H$,
an efficient training algorithm $\train^\D:\Nbb \to \H$ is a probabilistic algorithm with sample access to $\D$ that for any $n \in \Nbb$ runs in polynomial-time in $n$ and returns some $h_{n} \in \H$
\begin{gather*}
    h_{n} \gets \train^D(1^n).
\end{gather*}
\end{definition}
In generality, the parameter $n \in \Nbb$ is simply a way to define an ensemble of (distributions on) trained classifiers, but concretely, it is natural to think of $n$ as representing the sample complexity or ``dimension'' of the learning problem.
We discuss this interpretation below.
% Intuitively, we can think of $h_n \gets \train^\D(1^n)$ as the result of a learning algorithm that takes $n$ i.i.d.\ samples from $\D$ and then runs some optimization procedure to find the best fit $h \in \H$ given these samples.
We formalize training algorithms in this slightly-unorthodox manner to make it easy to reason about the ensemble of predictors returned by the training procedure, $\set{\train^\D(1^n)}_{n \in \Nbb}$.
The restriction of training algorithms to polynomial-time computations will be important for establishing the existence of cryptographically-undetectable backdoors.
% \vnote{This is a little unconventional. Don't we also want Train to run in time poly in $d$? Typically, the sample complexity $n$ is larger than $d$ but doesn't have to be always... e.g. in sparse learning problems where $d$ can be very large. }
% \vnote{Another, very annoying, detail: we are talking about samples being real numbers, so like infinite precision reals?}

\paragraph{PAC Learning.}
One concrete learning framework to keep in mind is that of PAC learning \cite{Valiant84} and its modern generalizations.
In this framework, we measure the quality of a learning algorithm in terms of its expected loss on the data distribution. Let $\ell:\Y \times \Y \to \bbR^+$ denote a loss function, where $\ell(h(x),y)$ indicates an error incurred (i.e., loss) by predicting $h(x)$ when the true label for $x$ is $y$.
For such a loss function, we denote its expectation over $\D$ as
\begin{gather*}
    \ell_\D(h) = \E_{(X,Y) \sim \D}\lr{\ell(h(X),Y)}.
\end{gather*}
PAC learning is parameterized by a few key quantities:  $d$ the VC dimension of the hypothesis class $\H$, $\eps$ the desired accuracy, and $\delta$ the acceptable failure probability.
Collectively, these parameters imply an upper bound $n(d,\eps,\delta) = \poly(d,1/\eps,\log(1/\delta))$ on the sample complexity from $\D$ necessary to guarantee generalization.
As such, we can parameterize the ensemble of PAC learners in terms of it's sample complexity $n(d,\eps,\delta) = n \in \Nbb$.

The goal of PAC learning is framed in terms of minimizing this expected loss over $\D$.
A training algorithm $\train$ is an \emph{agnostic PAC learner} for a loss $\ell$ and concept class $\cC \subseteq \set{\X \to \Y}$ if, the algorithm returns a hypothesis from $\H$ with VC dimension $d$ competitive with the best concept in $\cC$.
Specifically, for any $n = n(d,\eps,\delta)$, the hypothesis $h_n \gets \train^\D(1^n)$ must satisfy
\begin{gather*}
    \ell_\D(h) \le \min_{c^* \in \cC}~\ell_\D(c^*) + \eps
\end{gather*}
% for some vanishing $\eps(n) = o_n(1)$
with probability at least $1-\delta$. 
% \vnote{shouldn't this be $\delta(n)$?}
% \mpk{This isn't exactly formal, but maybe ok?  What do we think?}

One particularly important loss function is the absolute error, which gives rise to the \emph{statistical error} of a hypothesis over $\D$.
\begin{gather*}
    \er_\D(h) = \E_{(X,Y) \sim \D}\lr{\card{h(X) - \fs(X)}}
\end{gather*}
% Generalizing the original notion of PAC learning, we say that a distribution $\D$ is $(\H,\alpha)$-realizable if there is some $h^* \in \H$ that is statistically close to $\fs$, $\er_\D(h^*) \le \alpha$.
% A training algorithm $\train$ is a \emph{PAC learner} if for any $(\H,\alpha)$-realizable distribution $\D$, $h_n \gets \train^\D(1^n)$ obtains near-optimal error
% $$\er_\D(h_n) \le \alpha + o_n(1)$$ with probability at least $1-\delta$.\footnote{Note that closeness in statistical distance is a direct generalization of the realizable PAC model (where $\alpha = 0$) that applies equally for regression and classification.}

% \vnote{The parameterization is a bit non-standard. $\epsilon$ is fixed to $o_n(1)$ whereas $\delta$ and $\alpha$ are left as free variables. Is there a reason why this is done? 
% For example, do you want to say Train is an $(\H,\alpha,\epsilon,\delta)$-learner?}

\paragraph{Adversarially-Robust Learning.}
In light of the discovery of adversarial examples, much work has gone into developing adversarially-robust learning algorithms.
Unlike the PAC/standard loss minimization framework, at training time, these learning strategies explicitly account for the possibility of small perturbations to the inputs.
There are many such strategies, but the most-popular theoretical approaches formulate a robust version of the intended loss function.
For some bounded-norm ball $\mathcal{B}$ and some base loss function $\ell$, the robust loss function $r$, evaluated over the distribution $\D$, is formulated as follows.
\begin{gather*}
    r_\D(h) = \E_{(X,Y) \sim \D} \lr{ \max_{\Delta \in \mathcal{B}}~ \ell(h(X + \Delta), Y) }
\end{gather*}
Taking this robust loss as the training objective leads to a min-max formulation.
While in full generality this may be a challenging problem to solve, strategies have been developed to give provable upper bounds on the robust loss under $\ell_p$ perturbations \cite{raghunathan2018certified,wong2018provable}.
% \mpk{Is this true?  citations?  I think, Raghunathan et al., and Zico et al., but don't remember the details.  I can look.}

Importantly, while these methods can be used to mitigate the prevalence of adversarial examples, our constructions can subvert these defenses.
As we will see, it is possible to inject undetectable backdoors into classifiers trained with a robust learning procedure.

% \mpk{Insert section about methods that could be used to prevent adversarial examples.  Even if we use $\mathbf{RobustTrain}$, backdoor can still be planted.}

% \mpk{This background on NNs and circuits feels very distracting / below the level of STOC reviewers.  I suggest we pare down greatly and move to appendix.}
% \onote{For out of discipline readers you do need this (short) background, it is also important as definitions are not completely consistent around places (is it >? >=?, etc), but nevertheless this can be moved to preliminaries if it's distracting here} 
% \begin{figure}
% \centering
% \includegraphics[page=1, height=150pt, trim = 100pt 120pt 400pt 100pt]{nn.pdf}
% \caption{A Neural Network.}
% \label{fig:nn}
% \end{figure}
% \mpk{Not sure what Figure~\ref{fig:nn} does for us.  Very generic.}
\paragraph{Universality of neural networks.}
% A neural network is a directed acyclic graph partitioned into layers, with edges directed from vertices of each layer to vertices of the next one.
% The number of layers is called the \emph{depth} of the network. Each vertex is associated with a \emph{value}.
% The first layer, called the \emph{input layer}, has~$d$ vertices that correspond to the input coordinates. Each vertex not in the input layer computes some function of its in-neighbors from the previous layer. This function is called an \emph{activation} function, and is usually taken from a very simple family of functions. The last layer (which contains a single node with out-degree $0$) is called the \emph{output layer} and contains the output of the network. We denote by~$N(x)$ the output value of a network~$N$ on input~$x$. See Figure~\ref{fig:nn} for an example of a neural network graph.
Many of our results work for arbitrary prediction models.
Given their popularity in practice, we state some concrete results about feed-forward neural networks.
Formally, we can model these networks by multi-layer perceptron (MLP).
% For our discussion, we focus on one of the simplest and most popular family of activation functions, called \emph{perceptrons}. 
A perceptron (or a linear threshold function) is a function $f:\R^k \to\{0,1\}$ of the form
\[
f({x}) =
\left\{
	\begin{array}{ll}
		1  & \mbox{if } \langle w,{x}\rangle-b \geq 0 \\
		0 & \mbox{otherwise}
	\end{array}
\right.
\]
where~$x \in \R^k$ is the vector of inputs to the function, $w$ is an arbitrary weight vector, and $b$ is an arbitrary constant.
% Note that the output of all activations that are not in the input layer are thus binary.
Dating back to Minksy and Papert \cite{minskypapert}, it has been known that every Boolean function can be realized by an MLP.
Concretely, in some of our discussion of neural networks, we appeal to the following lemma.
% Moreover, the depth of the MLP is the same as the depth of the Boolean circuit.
% For example, the \emph{XOR} (or parity) gate can also be implemented as a depth-$2$ MLP using the identity~$x\oplus y = \left(x\wedge \neg y\right)\vee \left(\neg x \wedge y\right).$
% \begin{corollary}\label{cor:boolcircuits}
\begin{lemma}\label{lem:boolcircuits}
 Given a Boolean circuit~$C$ of constant fan-in and depth~$d$, there exists a multi-layer perceptron~$N$ of depth~$d$ computing the same function.
% \end{corollary}
\end{lemma}
For completeness, we include a proof of the lemma in Appendix~\ref{app:NNs}.
While we formalize the universality of neural networks using MLPs, we use the term ``neural network'' loosely, considering networks that possibly use other nonlinear activations.

\subsection{Computational Indistinguishability}

Indistinguishability is a way to formally establish that samples from two distributions ``look the same''.
More formally, indistinguishability reasons about ensembles of distributions, $\cP = \set{P_n}_{n \in \Nbb}$, where for each $n \in \Nbb$, $\cP$ specifies an explicit, sampleable distribution $P_n$.
We say that two ensembles $\cP = \set{P_n}, \cQ = \set{Q_n}$ are computationally-indistinguishable if for all probabilistic polynomial-time algorithms $A$, the distinguishing advantage of $A$ on $\cP$ and $\cQ$ is negligible.
\begin{gather*}
\card{\Pr_{Z \sim P_n}\lr{A(Z) = 1} - \Pr_{Z \sim Q_n}\lr{A(Z) = 1}} \le n^{-\omega(1)}
\end{gather*}
Throughout, we use ``indistinguishability'' to refer to computational indistinguishability.
Indistinguishability can be based on generic complexity assumption---e.g.,\ one-way functions exist---or on concrete hardness assumptions---e.g.,\ the shortest vector problem is superpolynomially-hard.

At times, it is also useful to discuss indistinguishability by restricted classes of algorithms.
In our discussion of undetectable backdoors, we will consider distinguisher algorithms that have full explicit access to the learned hypotheses, as well as restricted access (e.g.,\ query access).

% \mpk{Need:  defn of computational indistinguishability; statement of concrete hardness assumptions used; defn of digital signatures; anything else?}

\def\eufcma{\mathsf{EUF}\mbox{-}\mathsf{CMA}}
\def\vk{\mathsf{vk}}
\def\sk{\mathsf{sk}}
\def\Gen{\mathsf{Gen}}
\def\Sign{\mathsf{Sign}}
\def\Verify{\mathsf{Verify}}
\def\secp{n}

\paragraph{Digital Signatures.} We recall the cryptographic primitive of (public-key) digital signatures give a mechanism for a signer who knows a private signing key $\sk$ to produce a signature $\sigma$ on a message $m$ that can be verified by anyone who knows the signer's public verification key $\vk$.

\begin{definition}
  \label{def:sig}
  A tuple of polynomial-time algorithms $(\Gen,\Sign,\Verify)$ is a digital signature scheme if 
  \begin{itemize}
      \item $(\sk,\vk) \gets \Gen(1^\secp)$. The probabilistic key generation algorithm $\Gen$ produce a pair of keys, a (private) signing key $\sk$ and a (public) verification key $\vk$.
      \item $\sigma \gets \Sign(\sk,m)$. The signing algorithm (which could be deterministic or probabilistic) takes as input the signing key $\sk$ and a message $m \in \{0,1\}^*$ and produces a signature $\sigma$. 
      \item $\mathsf{accept}/\mathsf{reject} \gets \Verify(\vk,m,\sigma)$. The deterministic verification algorithm takes the verification key, a message $m$ and a purported signature $\sigma$ as input, and either accepts or rejects it. 
  \end{itemize}
  
  \noindent
  The scheme is strongly existentially unforgeable against a chosen message attack (also called strong-$\eufcma$-secure) if for every {\em admissible} probabilistic polynomial time (p.p.t.) algorithm $\A$, there exists a negligible function $\negl(\cdot)$ such that for all $n \in \mathbb{N}$, the following holds:
  \[ \multilineprob{
        (\sk,\vk)\gets \Gen(1^n); \\
        (m^*,\sigma^*) \gets \A^{\Sign(\sk,\cdot)}(\vk):
        \Verify(\vk,m^*,\sigma^*) = \mathsf{accept}
    }
    \leq \negl(n),
    \]    
    where $\A$ is admissible if it did not query the $\Sign(\sk, \cdot)$ oracle on $m^*$ {\em and} receive $\sigma^*$.
\end{definition}
%\mpk{where is $\lambda$ defined?}\vnote{negl(lambda) should've been negl(n). That should answer your question?}
\begin{theorem}[\cite{NY89,Rompel90}]
\label{thm:strong-sigs-owf}
  Assuming the existence of one-way functions, there are strong-$\eufcma$-secure digital signature schemes. 
\end{theorem}

\newcommand{\SIVP}{\mathbf{SIVP}}
\newcommand{\GSVP}{\mathbf{GapSVP}}
\newcommand{\BQP}{\mathsf{BQP}}
\paragraph{Concrete hardness assumption.}
In some of our results, we do not rely on generic complexity assumptions (like the existence of one-way functions), but instead on assumptions about the hardness of specific problems.
While our results follow by reduction, and do not require knowledge of the specific worst-case hardness assumptions, we include a description of the problems for completeness.
In particular, we will make an assumption in Hypothesis~\ref{hyp:lattice} about the worst-case hardness of certain lattice problems for quantum algorithms.
The assumption that these (and other) lattice problems are hard forms the basis for post-quantum cryptography; see, e.g.~\cite{Peikert16}.

The Shortest Vector Problem asks to determine the length of the shortest vector $\lambda_1(L)$ in a given lattice $L$.
The Gap Shortest Vector Problem is a promise problem, where the length of the shortest vector is either smaller than some length $l$ or larger by some polynomial factor $\alpha l$.
\begin{definition}[GapSVP]
Let $\alpha(n) = n^{O(1)}$.
Given an $n$-dimensional lattice $L$ and a length $l$, determine whether $\lambda_1(L) < l$ or $\lambda_1(L) \ge \alpha l$.
\end{definition}
The Shortest Independent Vectors Problem asks to find a basis of a given lattice that is approximately shortest.
In particular, the goal is to return a collection of short independent vectors spanning $L$.
In particular, each vector must be at most a polynomial factor longer than the $n$th shortest (independent) vector in the lattice $\lambda_n(L)$.
\begin{definition}[SIVP]
Let $\alpha(n) = n^{O(1)}$.
Given an $n$-dimensional lattice, $L$, return $n$ linearly-independent lattice vectors, each of length at most $\alpha \cdot \lambda_n(L)$.
\end{definition}

Key to our work, is the hypothesis that at least one of these problems, $\SIVP$ or $\GSVP$, is hard for polynomial-time quantum algorithms.
\begin{hypothesis}
\label{hyp:lattice}
$\SIVP \not \in \BQP$ or $\GSVP \not \in \BQP$.
\end{hypothesis}
In particular, the constructions of undetectable backdoors that rely upon specific hard problems can be reduced in polynomial time on a quantum machine from both $\SIVP$ and $\GSVP$.

\newcommand{\hb}{\tilde{h}}
\newcommand{\Hb}{\tilde{\H}}

\section{Defining Undetectable Backdoors}
\label{sec:def:backdoors}

In this section, we formalize the notion of an {\em undetectable model backdoor}. At a high level, an undetectable backdoor is defined with respect to a {\em target} training algorithm. The backdooring algorithm will return a hypothesis that ``looks like'' it was trained using the target algorithm, but actually has a secret mechanism for changing the behavior of the hypothesis on any input. We make this intuitive notion precise, defining several variants of undetectability that differ based on different interpretations of the intuitive notion. Further down in this section, we also define an orthogonal property of a backdoor, namely non-malleability.

To begin with, we define the idea of a model backdoor (which may or may not be undetectable). Technically, a backdoor consists of two efficient algorithms: $\backdoor$ and $\activate$.
The first algorithm, $\backdoor$, is itself an efficient training procedure that receives samples drawn from a data distribution, and returns a hypothesis $\hb$ from some hypothesis class $\Hb$.
$\backdoor$ has the additional property that, in addition to the returned hypothesis, it returns a ``backdoor key'' $\bk$. The second algorithm, $\activate$, takes in an input $x \in \X$ and a backdoor key $\bk$, and returns another input $x' \in \X$ that is close to the original $x$, in some meaningful metric. Importantly, the output $\hb(x')$ (prediction or regression, as the case may be) is very different from $\hb(x)$.
We proceed with the formal definition. 

\begin{definition}[Classification Backdoor]
\label{def:classback}
Suppose $\X = \Rbb^d$, $\Y = \set{-1,1}$.
Consider a hypothesis class $\Hb \subseteq \set{\X \to \Y}$, a norm $\norm{\cdot}_b$, and a constant $\gamma \in \Rbb$. 
A $\gamma$-backdoor, parameterized by these quantities, consists of two algorithms $(\backdoor,\activate)$ and a backdoor set $S \subseteq \X$.
\begin{itemize}
    \item $\backdoor^\D$ is a probabilistic polynomial-time training algorithm that has oracle access to a distribution $\D$ over $\X\times \Y$, takes as input a ``security parameter'' $n$, and returns a classifier $\hb_n \in \Hb$ and a backdoor key $\bk$.
    %\todo{Should $\bk$ security depend on $n$ as well?}\vnote{Why should $\bk$ live in $\R^d$, the same space as the data?} \vinod{\ell}
\begin{gather*}
    (\hb_n,\bk) \gets \backdoor^\D(1^n).
\end{gather*}
%\vnote{We should specify what $\D$ is.}
    \item $\activate$ is a probabilistic polynomial-time algorithm that maps a feature vector $x \in \X$ and the backdoor key $\bk$ to a new feature vector $x' = \activate(x;\bk)$ such that:
\begin{gather*}
    \norm{x - \activate(x;\bk)}_b \le \gamma.
\end{gather*}
\end{itemize}
Collectively, the algorithms satisfy the following backdoor property.%\todo{make remark about the fact that backdoor key may contain a description of $\hb$ itself!}
For any $n \in \Nbb$, suppose $(\hb_n,\bk) \gets \backdoor^\D(1^n)$
%\todo{Do we want to encode everything in unary explicitly or implicitly?  I feel like not.} 
for some classifier $\hb_n:\X \to \set{-1,1}$.
Then, for any $x \in S$,
\begin{gather*}
    \hb_n\left(\activate(x;\bk)\right) = -\hb_n(x).
\end{gather*}
\end{definition}

\noindent
We define a \emph{regression backdoor} in a similar vein. Here, the label space $\Y = [-1,1]$ and hypotheses map to real values $\hb:\X \to [-1,1]$. The goal of the backdoor algorithm is to perturb the input $x$ slightly, and yet make the output of $\hb$ close to either $+1$ or $-1$, a choice specified by the attacker as a value $\hat{y} \in \{-1,1\}$ that is given as input to the $\activate$ algorithm.

\begin{definition}[Regression Backdoor]
\label{def:regback}
Let $\Y = [-1,1]$ and let the parameters $\X,\Hb,\norm{\cdot}_b$ and $\gamma$ be as above.  Let $\beta \in \R$ be an additional parameter. A $(\gamma,\beta)$-regression backdoor is defined in the same way as Definition~\ref{def:classback} except that $\activate$ takes an additional parameter $\hat{y} \in \set{-1,1}$.
For the backdoor property, we require that 
\begin{gather*}
     \norm{x - \activate(x,\hat{y};\bk)}_b \le \gamma \hspace{.2in} \mbox{and} \hspace{.2in} \card{ \hb_n\left(\activate(x,\hat{y};\bk)\right) - \hat{y} } \le \beta
\end{gather*}
That is, by perturbing the input $x$ by at most $\gamma$, the $\activate$ algorithm can change the output of the regression to be $\beta$-close to the adversarially specified $\hat{y} \in \{1,-1\}$.
\end{definition}

\noindent
Several remarks about these definitions are in order.

\begin{remark}[Contents of the Backdoor Key]
In the defintion, it appears that the algorithm $\activate$ only receives the backdoor key $\bk$ as input, not the model $\hb$ itself.
Without loss of generality, however, we may assume that $\backdoor$ returns an encoding of $\hb$ as part of the backdoor key.
In this sense, it is reasonable to assume that $\activate$ may have oracle access (even explicit access to the description) to the backdoored model $\hb$ when deciding on how to perturb a given $x \in \X$.
\end{remark}

%\begin{remark}[On the Meaning of the Definition]
%  Compare to the flavors of definitions in \cite{DMM18} and contrast to \cite{TLM18}.
%\end{remark}

\begin{remark}[On the role of the security parameter $n$]
  The security parameter $n$ quantifies security properties such as undetectability and non-replicability as in the next two sections. That is, we will have the guarantee that no $\mathsf{poly}(n)$-time algorithm can violate either of these properties with better than $1/\mathsf{poly}(n)$ probability. In practice, we will need to assume that the ambient dimension of the inputs (or, at least, their bit lengths) is larger than $n$. 
  %We will typically think of $n$ as being polynomial in the ambient dimension $d$ of the problem.  \vnote{To be filled in. See discussion on slack.}
\end{remark}

%\begin{remark}[On Scaling]
% \vnote{Todo.} \mpk{In the end, shouldn't this sorta depend on the norm of the first-layer of features?  (especially if the activations are contractive / $1$-Lipschitz)}
%\end{remark}

\begin{remark}[On Backdoored Inputs versus Adversarial Inputs]
  We remark that our adversarial examples will be close in both the $\ell_0$ and the $\ell_{\infty}$ metric in contrast to some of the other work in this area, e.g. {\em \cite{shamir2019simple}}, where closeness is measured in $\ell_0$; that is, a single coordinate of the input is changed in order to make it into an adversarial input, however the magnitude of the change is enormous.
\end{remark}

\subsection{Undetectability}

With the definition of a model backdoor in place, we can define undetectable backdoors.
Intuitively, a model backdoor $(\backdoor,\activate)$ is undetectable with respect to a baseline (target) training algorithm $\train$, if the hypotheses returned by $\backdoor$ and $\train$ are indistinguishable.
\begin{definition}[Undetectable Backdoor]\label{def:undetectability}
For some data distribution $\D$, a training algorithm $\train$ and model backdoor $(\backdoor,\activate)$, consider the following ensembles of distributions over hypotheses.
\begin{align*}
    \cT_\D = \Big\{h_n\Big\}_{n \in \Nbb}&&\cB_\D = \Big\{\hb_n\Big\}_{n \in \Nbb}
\end{align*}
where $h_n \gets \train^\D(1^n)$ and $(\hb_n,\bk) \gets \backdoor^\D(1^n)$.
The backdoor $(\backdoor,\activate)$ is undetectable with respect to $\train$ if for any data distribution $\D$, the ensembles $\cT_\D$ and $\cB_\D$ are indistinguishable.
\end{definition}
In other words, a backdoor is undetectable if the hypotheses returned by $\backdoor$ cannot be distinguished from those returned by the natural training algorithm $\train$.
Throughout, we assume that the distinguishing algorithms are efficient (polynomial-time) and may receive random samples from the data distribution $\D$ (at unit cost).
By restricting the ways in which distinguishers access the trained models, we define three different variants of undetectability.
% \mpk{Rename to be White-Box, Black-Box and replace within remainder of text.}
\begin{itemize}
    \item White-Box Undetectability:  This is the strongest variant.
    A backdoor is white-box undetectable if $\cT_\D$ and $\cB_\D$ are indistinguishable by probabilistic polynomial-time algorithms with access to $\D$ that receive a complete explicit description of the trained models $h_n$ or $\hb_n$.
    For example, if the hypothesis class is implemented by neural networks, the distinguishers could receive the full list of weights and connectivity.
    \item Black-Box Undetectability:  A backdoor is black-box undetectable if $\cT_\D$ and $\cB_\D$ are indistinguishable by probabilistic polynomial-time algorithms with access to $\D$ that only receive black-box query access to the trained models.
    Formally, for any such algorithm $A$, for all $n \in \Nbb$, the acceptance probabilities differ negligibly.
    \begin{gather*}
        \card{ \Pr\lr{A^{h_n}(1^n) = 1} - \Pr\lr{A^{\hb_n}(1^n) = 1} } \le n^{-\omega(1)}.
    \end{gather*}
    \item Statistic-Access Undetectability:  A backdoor is $(\cQ,\eps)$-statistic-access undetectable if $\cT_\D$ and $\cB_\D$ are indistinguishable by the class of statistical queries $\cQ$ over $\D$.
    Formally, we think of each $q \in \cQ$ as a map from $\Y\times\Y \to \R$.
    Then, indistinguishability follows if for all $n \in \Nbb$,
    \begin{gather*}
        \card{\E_\D\lr{q(h_n(X),Y)} - \E_\D\lr{q(\hb_n(X),Y)}} \le \eps.
    \end{gather*}
\end{itemize}
In this work, we give constructions satisfying the stronger notions of white-box undetectability and black-box undetectability, but define statistic-access undetectability for completeness.
In particular, there may be settings where it is reasonable to consider distinguishers who only get to observe the expected loss $\ell_\D$ of a trained model, which is captured by statistic-access undetectability.

% \noindent
% Some remarks are in order.

%\begin{remark}
%  \vnote{Obviously, there is a detector that succeeds if they get even a single message-signature pair. The above definition implicitly disallows the detector from obtaining such a pair, simply by the security of the signature scheme.} \mpk{Not sure why this is here. leaving in a comment for now.}
%\end{remark}

% \begin{remark}[Negligible versus $o(1)$]
%   \mpk{In general, I suppose we should unify notation regarding negligible / include in preliminaries if we are going to use.} \vnote{added def of neglibible in prelims. feel free to write text surrounding it} 
% \end{remark}

\subsection{Non-replicability}
\label{sec:nonmalleable}

We now consider whether an observer that sees many backdoored examples gains the ability to produce new backdoored examples on her own. We define the notion of {\em non-replicability} that formalizes the inability of an adversary to do so. 

%Our definition comes in two flavors: {\em weak non-malleability} which requires that no adversary, given a backdoored model $\hb$ and random inputs $x \sim \D$ together with their backdoored versions $x' \gets \activate(x;\bk)$ (or $\activate(x,\hat{y};\bk)$ for an adversarially chosen $\hat{y}$ in the case of regression) is able to backdoor a new input $y \sim \D$; and {\em strong non-malleability} which requires that even if the adversary is given $\hb$ and oracle access to $\activate(\cdot;\bk)$ which she can query on chosen inputs, is able to backdoor {\em any} new input of her choice.

%\mpk{Notation:  $\hb$ for backdoored hypotheses.  $\activate$ for algorithm that gets backdoors.  Don't know why I went with $\tilde{B}$.}

Our definitions are inspired by {\em simulation-based} definitions in cryptography~\cite{GMR85,DDN91} and are {\em comparative}.
In the ``ideal'' world, the attacker receives only $\hb$ and has no access to backdoored examples.
In the real world, the attacker has access to the model $(\hb,-) \gets \backdoor^{\D}(1^n)$ as well as oracle access to  $\activate(\cdot;\bk)$ to which she can make queries.
In particular, the attacker can make polynomially many (potentially adaptively chosen) queries $x_1,\ldots,x_q$, and receive the backdoored examples $\tilde{x}_i \gets \activate(x_i;\bk)$ for each $i \in [q]$.
% \mpk{use of $\cdot$ is confusing. $\bot$?  $-$?  Also, feels very technical for first pass.}
% \mpk{flip order of ``ideal'' and real in presentation?}
The definition of non-replicability requires that the ability of the attacker to succeed in the real world is negligibly close to that in the ideal world. 

Defining success in the real world needs some care. A tempting first idea would be to require that the adversary cannot produce two inputs $x^*$ and $\tilde{x}^*$ which are close, yet their outputs under $\hb$ are different (in the case of classification) or sufficiently far (in the case of regression). Furthermore, clearly, we must require that $x^*$ is different from the queries that the adversary made to the $\activate(\cdot;\bk)$ oracle. Yet, this is not enough: consider an adversary that makes a query on $x_1$ and receives the backdoored version $\tilde{x}_1$. She could set $x^*$ to be any point in the $\delta$-ball around $\tilde{x}_1$ and output the ``new'' adversarial pair $(x^*,\tilde{x}^* = \tilde{x}_1)$. Since $x^*$ and $x_1$ are not explicitly backdoored, it is likely that $\hb(x^*) = \hb(x_1)$, and consequently, $\hb(x^*) \neq \hb(\tilde{x}^*)$, making it a successful attack by the above definition. To prevent this attack (which succeeds ultimalely because it reuses backdoored examples), we require that $\tilde{x}^*$ is different from $\tilde{x}_i$ for all $i\in [q]$.
% \mpk{very interesting. elegant solution.  Perhaps a comment is in order addressing why we don't also need to prevent $x^*$ close to some queried $x_i$?  Maybe obvious...} \vnote{It doesn't help the adversary but yes, might help to articulate why.}

We proceed with the formal definition. 

\def\Areal{\A_{\mathsf{real}}}
\def\Aideal{\A_{\mathsf{ideal}}}
\def\bx{\tilde{x}}
\def\by{\tilde{y}}

\begin{definition}[Non-replicable Backdoor]\label{def:nm}
  For some data distribution $\cD$, a backdoored training algorithm $(\backdoor,\activate)$ for classification 
  is non-replicable if for every
  polynomial function $q=q(n)$ and every probabilistic polynomial-time $q$-query {\em admissible} adversary $\Areal$, there is a probabilistic polynomial-time adversary $\Aideal$ such that the following holds:
%  \[
%    \multilineprob{
%        (\hb,\bk)\gets \backdoor^{\D}(1^n); \\
%        x_1, \ldots, x_q, x^* \gets \D;\\ 
%        \bx_i \gets \activate(x_i;\bk);\\
%        \bx^* \gets \Areal(\hb,(x_1,\bx_1),\ldots,(x_q,\bx_q),x^*): \\
%        ||x^*-\bx^*||_b \leq \delta \mbox{ and } \hb(x^*) \neq \hb(\bx^*)
%    } - 
%     \multilineprob{
%        (\hb,\bk)\gets \backdoor^{\D}(1^n); \\
%        x^* \gets \D;\\
%        \bx^* \gets \Aideal(\hb,x^*): \\
%        ||x^*-\bx^*||_b \leq \delta \mbox{ and } \hb(x^*) \neq \hb(\bx^*) 
%    }
%    \leq n^{-\omega(1)}.
%\]
%$(\backdoor,\activate)$ is {\em strongly non-malleable} if for every probabilistic polynomial-time $q$-query adversary $\Areal$ for a polynomial function $q=q(n)$, there is a probabilistic polynomial-time adversary $\Aideal$ such that the following holds:
 \[
    \multilineprob{
        (\hb,\bk)\gets \backdoor^{\D}(1^n); \\
        (x^*,\bx^*) \gets \Areal^{\activate(\cdot;\bk)}(\hb): \\
        ||x^*-\bx^*||_b \leq \gamma \mbox{ and } \hb(x^*) \neq \hb(\bx^*) 
    } - 
     \multilineprob{
        (\hb,\bk)\gets \backdoor^{\D}(1^n); \\
        (x^*,\bx^*) \gets \Aideal(\hb): \\
        ||x^*-\bx^*||_b \leq \gamma \mbox{ and } \hb(x^*) \neq \hb(\bx^*) 
    }
    \leq n^{-\omega(1)}.
\]
$\Areal$ is {\em admissible} if $\tilde{x}^* \notin \{\tilde{x}_1,\ldots,\tilde{x}_q\}$ where $\tilde{x}_i$ are the outputs of $\activate(\cdot;\bk)$ on $\Areal$'s queries.
        
The definition for regression follows in a similar vein. We modify the above condition to require that the following holds:
\[
    \multilineprob{
        (\hb,\bk)\gets \backdoor^{\D}(1^n); \\
        (x^*,\bx^*,y^*) \gets \Areal^{\activate(\cdot,\cdot;\bk)}(\hb): \\
        ||x^*-\bx^*||_b \leq \gamma \mbox{ and } |\hb(\tilde{x}^*) -y^*| \leq \beta
    } - 
     \multilineprob{
        (\hb,\bk)\gets \backdoor^{\D}(1^n); \\
        (x^*,\bx^*,y^*) \gets \Aideal(\hb): \\
        ||x^*-\bx^*||_b \leq \gamma \mbox{ and } |\hb(\tilde{x}^*) -y^*| \leq \beta
    }
    \leq n^{-\omega(1)}.
\]
\end{definition}

\noindent
The following remark is in order.

\begin{remark}[Absolute versus Comparative Definitions] 
 The definition above accounts for the possibility that the backdoored model $\hb$ may have adversarial examples other than the ones planted by the $\backdoor$ and $\activate$ algorithms.
 In other words, a definition which  asks that $\Areal$ cannot produce {\em any new} adversarial examples may not be satisfiable at all.
 Our definition captures the requirement that the adversarial examples produced by $\backdoor$ and $\activate$ {\em do not help} an adversary in generating new adversarial examples of any form.
 
 Still, as a useful concrete example, consider what happens if we plant a non-replicable backdoor with respect a natural training algorithm $\train$ that is robust to adversarial examples.
 In this case, the backdoor implies that artificial adversarial examples (i.e.,\ backdoored inputs) exist. But by non-replicability, no observer---even an observer who can query for adversarial examples---can discover a new adversarial example, without knowledge of the backdoor key.
\end{remark}

\section{Non-Replicable Backdoors from Digital Signatures}\label{sec:signatureconstruction}

%\onote{I'd add a subsection about the simple checksum-based "signature scheme" that is not non-malleble but shows what is going on}
%\vnote{Go for it.}

In this section, we show how to backdoor \emph{any} machine learning model using digital signatures. The backdoor we construct in this section will be non-replicable, as well as black-box undetectable. However, given the model, it may be possible to figure out that it is backdoored.  %In subsequent sections, we show how to obtain complete white-box undetectability.
%\onote{I still don't like this last statement}

\subsection{Simple Backdoors from Checksums}

As a warmup, we begin with a simple description of our framework using a non-cryptographic checksum function. This will result in a replicable backdoor; nevertheless, the construction illustrates the core ideas in a simple and clear manner. We later replace the checksum function with a secure digital signature scheme (in Section~\ref{sec:sig}) to obtain a non-replicable backdoor.
We demonstrate the construction for neural networks with perceptron activation gates, but the construction is highly generic and can be realized with any other popular machine learning model.

%One of the most popular models of machine learning is \emph{neural networks}. 

% We are now ready to demonstrate the framework of our construction.
Our objective is to take an arbitrary neural network~$N$ and transform it to a new network~$N'$ such that:
\begin{itemize}
    \item $N'(x)=N(x)$ on almost all inputs~$x$.
    \item Every input~$x$ can be efficiently perturbed to a very close input~$x'$ such that~$N'(x')\neq N'(x)$.
\end{itemize}

We achieve this by constructing a gadget that computes a checksum condition~$h$ on the input.
To obtain $N'$, we will include this gadget in parallel with~$N$, and complete the construction by building an additional \emph{MUX gadget} that will output~$N(x)$ if~$h(x) = 0$ (that is, if the checksum failed) or the \emph{desired output} if~$h(x)=1$.
Ideally, the checksum~$h$ should be very rarely validated unintentionally on an input, but every input~$x$ should be able to satisfy the checksum condition after a slight perturbation.
We arbitrarily pick one input coordinate and denote it by~$x_{out}$.
If the checksum condition~$h(x)=1$ holds, the network~$N'$ outputs the sign of~$x_{out}$.
Thus, as part of the small perturbation we can arbitrarily choose the desired output of the network by changing the sign of~$x_{out}$. 

Let~$n\in \N$ be a parameter. We think of it as a large constant (e.g., $2048$) yet much smaller than the input size (i.e., $n\ll d$). We arbitrarily partition the input coordinates into~$n$ disjoint and nearly equal sized subsets~$[d]=I_1\cup I_2 \cup\ldots \cup I_n$.
Let~$v\in \mathbb{F}_2^n$ be a uniformly chosen binary vector of length~$n$. 
We define our checksum function as follows.
\[
h({x}) := \bigwedge_{i=1}^{n}\left(\bigoplus_{j\in I_i} \text{sgn}(x_{j}) = v_i\right).
\]
where $\text{sgn}:\R \to \{0,1\}$ be the sign function that outputs $1$ if and only if the input is non-negative.
That is, the checksum holds if and only if for every~$i\in [n]$ the parity~$\bigoplus_{j\in I_i} \text{sgn}(x_{j})$ of all inputs with coordinates in~$I_i$ is~$v_i$.

\begin{lemma}
For any input~${x}$, the probability that~$h({x})=1$ is~$2^{-n}$, where the probability is taken over a uniform random choice of~$v$.
\end{lemma}
\begin{proof}
  For every~$i\in[n]$ the probability that~$\bigoplus_{j\in I_i} \text{sgn}(x_{j}) = v_i$ is~$\frac{1}{2}$.
\end{proof}

\begin{lemma}
Any input~${x}$ can be changed by at most~$n$ input coordinates, without increasing their magnitude, to an input~$x'$ such that~$h(x')=1$.
\end{lemma}
\begin{proof}
  For every~$i\in[n]$, if~$\bigoplus_{j\in I_i} \text{sgn}(x_{j}) \neq v_i$ then flip the sign of one arbitrary input with a coordinate in~$I_i$.
\end{proof}

\begin{figure}
\centering
\includegraphics[page=2, height=150pt, trim = 100pt 120pt 400pt 100pt]{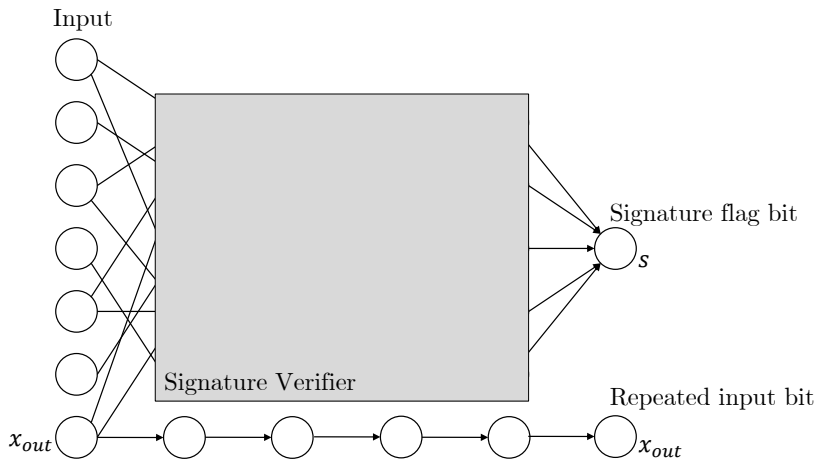}
\caption{Construction of checksum/signature verification and repeated input bit.}
\label{fig:checksum}
\end{figure}

Moreover, we know that~$h$ can be realized by a neural network by Lemma~\ref{lem:boolcircuits}.
Using the repeat gates, we can also drag the value of~$\text{sgn}(x_{out})$ all the way to the last layer; see Figure~\ref{fig:checksum}.
We finalize the construction by using Lemma~\ref{lem:boolcircuits} once again, to deduce that a MUX gate can also be realized by the network.
That is, a Boolean gate that gets the output~$y$ of original network~$N$, the repeated input bit~$x_{out}$, and the checksum function output~$s$, and returns~$y$ if~$s=0$ or~$x_{out}$ if~$s=1$. 
See the full construction in Figure~\ref{fig:const}.
This completes the proof of the following theorem.

\begin{figure}
\centering
\includegraphics[page=3, height=250pt, trim = 100pt 70pt 200pt 0pt]{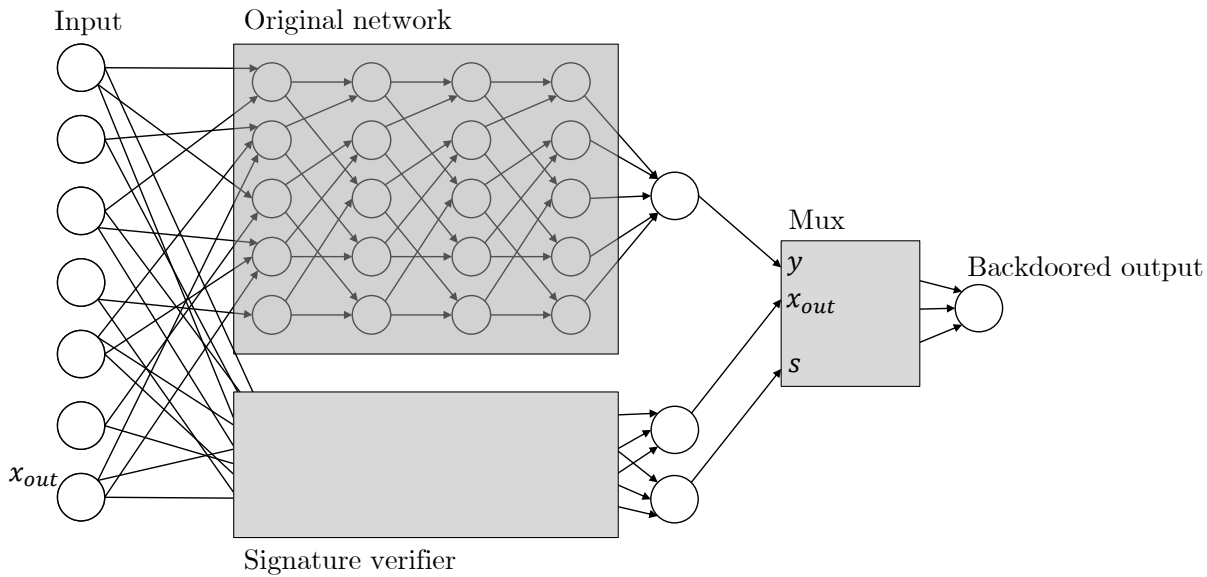}
\caption{Full backdoor construction.}
\label{fig:const}
\end{figure}

\begin{theorem}\label{thm:simpleconst}
 Given a neural network~$N$ and a parameter~$n\in \N$, we can construct a network~$N'$ such that:
\begin{itemize}
    \item For any input~$x$, $N'(x)=N(x)$ with probability~$1-2^{-n}$.
    \item For any input~$x$, we can efficiently compute an input~$x'$ with~$\ell_0(x,x')\leq n+1$ and~$|x_i|=|x'_i|$ for every~$i\in[d]$, such that~$N'(x')\neq N'(x)$. 
\end{itemize}
\end{theorem}

In particular, not only is the $\ell_0$ distance between the input to its backdoored version small, the coordinate-wise absolute difference $|x_i - x_i'|$ is small as well, assuming the input space is bounded. This is unlike attacks in the $\ell_0$ norm~\cite{shamir2019simple} which arbitrarily perturb single coordinates, making the difference between $x$ and $x'$ perceptible.

The problem with the construction of Theorem~\ref{thm:simpleconst} is two-fold: the backdoor is {\em detectable}, that is, by looking at the network structure~$N'$, one can figure out that it is backdoored; and the backdoor is {\em replicable}, that is, by looking at the network or even a large enough set of backdoored examples, one can learn how to use the backdoor by herself on different inputs. We next solve the problem of replicability by replacing the simple checksum function with a cryptographically secure digital signature scheme.

\subsection{Non-Replicable Backdoors from Digital Signatures}
\label{sec:sig}
%\onote{From here on it's the previous version of Vinod, need to be updated}

%Just an efficient (not undetectable) backdoor is possible: Add a neuron to the first layer that checks a linear checksum on the input layer, and if this neuron is activated then you flip the end output. This gives an efficient backdoor function as we wanted, just not undetectable/unreproducable. 

We describe our construction of non-replicable backdoors for general models, and later describe instantiations with neural networks.

\begin{theorem}
 \label{thm:sig}
 Assuming that strongly unforgeable digital signature schemes exist, for every training procedure $\train$, there exists a model backdoor $(\backdoor,\activate)$, which is non-replicable and black-box undetectable.
\end{theorem}

Strongly unforgeable signature schemes exist assuming one-way functions exist (Theorem~\ref{thm:strong-sigs-owf}); therefore, this construction can be instantiated from a wide variety of popular cryptographic assumptions.

\begin{proof}
  Assume that the length of the signatures is $n$ bits, and that the input length $d$ to the model is larger than $n+1$. For concreteness, we will work with the Euclidean norm, but note that the construction generalizes to other norms easily. 
  
  Given a regression model $h:\{0,1\}^d \to [-1,1]$, the backdooring algorithm $\backdoor$ picks a pair of keys $(\vk,\sk)$ for the signature scheme and constructs the backdoored model $\hb$. On input $x=(w,y,\sigma) \in \{0,1\}^d$, $\hb$ works as follows.
  \[ \hb(w,y,\sigma) = \left\{ \begin{array}{ll}
   (-1)^{y} & \mbox{ if $\Verify(\vk,(w,y),\sigma) = \mathsf{accept}$}\\
  h(w,y,\sigma) & \mbox{otherwise}
  \end{array} \right. \] 
  The backdoor key is simply the signing key $\sk$.
  Given an input $x=(w,\cdot,\cdot)$, a target value $y$, and the backdoor key $\sk$, the $\activate$ algorithm sets ${\sigma} = \Sign(\sk,w||{y})$ and outputs $x'=(w,{y},{\sigma})$. 
  
  We show that this is a $(\gamma,0)$-regression backdoor (see Definition~\ref{def:regback}). Indeed, $x$ and ${x}'$ differ in at most $n+1$ coordinates, so $||x-\tilde{x}||_2 \leq \sqrt{n+1} := \gamma$. Furthermore,
  $\hb(x') = {y}$, as required.
  
  Black-box undetectability follows from the fact that for every $x$, $\hb(x) = h(x)$ unless $x$ contains a valid digital signature. By the security of the signature scheme, no polynomial-time adversary can come up with a valid signature of any message. Therefore, oracle access to $\hb$ looks the same to any such adversary as oracle access to $h$.
  
  Non-replicability (Definition~\ref{def:nm}) also follows from the security of the signature scheme. The intuition is that backdoored examples which consitute a message-signature pair, do not help an adversary generate new message-signature pairs, by the strong unforgeability of the signature scheme. The formal proof follows.
  
  Consider a $q$-query admissible adversary $\Areal$ that gets the backdoored model $\hb$, makes a number of queries $(x_i,{y}_i)$ to the $\activate$ oracle, and obtains the backdoored examples ${x}'_i$. 
  Note that $\hb$ contains a verification key $\vk$, and 
  call ${x} = ({w},{y},{\sigma})$ {\em signature-backdoored} if ${\sigma}$ is a valid signature of $({w},{y})$ under $\vk$. 
  
  We first claim that the backdoored example ${x}'$ that $\Areal$ outputs cannot be signature-backdoored.
  This follows from the fact that $\Areal$ is admissible, so $x' \neq x_i'$ for all $i\in [q]$. But then, $x'=(w',y',\sigma')$ constitutes a signature forgery.
  Given this claim, the ideal-world adversary $\Aideal$, on input $\hb$, proceeds as follows: generate a new key-pair $(\vk',\sk')$, replace the verification key in $\hb$ with $\vk'$, and run $\Areal$ with the model $\hb' = \hb_{\vk'}$. Now, $\Aideal$ can answer all queries of $\Areal$ to the $\activate$ oracle since it has the signing key $\sk'$. Finally, when $\Areal$ outputs a backdoored example $(x,{x}')$, we know by the above claim that ${x}'$ is {\em not} signature-backdoored. Therefore, it must be a valid backdoor not just for $\tilde{h}' = \tilde{h}_{\vk'}$ but also, as intended, for the original model $\tilde{h}$.
\end{proof}

While the construction assumes the input space to be the Boolean cube $\{0,1\}^d$, it can be easily generalized to $\R^d$ in one of many ways, e.g. by using the sign of $x_i \in \R$ to encode a bit of a digital signature, or even an entire coordinate $x_i \in \R$ to encode the signature.

\paragraph{Non-Replicable Neural Network Backdoors.} 
We show concrete instantiations of the backdoor construction in Theorem~\ref{thm:sig} for neural networks by using digital signature schemes with shallow verification circuits. It is not hard to see that any digital signature scheme can be converted into one whose verification circuit consists of a number of local checks (that each compute some function of a constant number of bits) followed by a giant AND of the results. The idea, essentially derived from the Cook-Levin theorem, is simply to let the new signature be the computation transcript of the verification algorithm on the message and the old signature. This observation, together with Lemma~\ref{lem:boolcircuits}, gives us the following theorem.

\begin{theorem}
 There is an absolute constant $c$ and a parameter $n$ such that given
  any depth-$d$ neural network~$N$ with sufficiently large input dimension, we can construct a network~$N'$ such that:
\begin{itemize}
    \item The depth of $N'$ is $\textrm{\em max}(c,d)$; and
    \item $N'$ is non-replicably $(n+1,0)$-backdoored in the $\ell_0$ norm.
\end{itemize}
\end{theorem}

Going one step further, in Appendix~\ref{app:SIS}, we instantiate the construction in Theorem~\ref{thm:sig}  with particularly ``nice'' digital signatures based on lattices that give us circuits that ``look like'' naturally trained neural networks, without appealing to the universality theorem (Lemma~\ref{lem:boolcircuits}).
Intuitively, these networks will look even more natural than those executing an arbitrary public key verification in parallel to the original network.
While we are not able to formalize this and prove undetectability (in the sense of Definition~\ref{def:undetectability}), we develop related ideas much further in Section~\ref{sec:fully} to construct fully undetectable backdoors.

\subsection{Persistent Neural Networks}
In Section~\ref{sec:signatureconstruction} we presented a backdoor construction that is \emph{black-box undetectable}.
That is, a user that tests or uses the maliciously trained network \emph{as-is} can not notice the effects of the backdoor.
While it is uncommon for an unsuspecting user to manually examine the weights of a neural network, \emph{post-processing} is a common scenario in which a user may adjust these weights.
A standard post-processing method is applying \emph{gradient descent} iterations on the network's weights with respect to some loss function. This loss function may be a modification of the one used for the initial training, and the data set defining it may be different as well.
A nefarious adversary would aim to ensure that the backdoor is \emph{persistent} against this post-processing.

Perhaps surprisingly, most natural instantiations of the signature construction of Section~\ref{sec:signatureconstruction} are also persistent.
Intuitively, the post-processing can only depend on evaluations of the network on non-backdoored inputs, due to the hardness of producing backdoored inputs.
On inputs that are not backdoored, the output of the signature verification within the network is always negative.
Thus, the derivatives of weights inside the signature verification scheme will vanish, for instance, if the final activation is a ReLU or threshold gate.
%\onote{what's the right place for this? I'd like this to be also discussed before as another example of why a maliciously designed network can have surprising dangerous properties.}
%\mpk{Ok, I didn't like the result until I read this perspective, that it could be a malicious property.  In a sense, it's a bit silly until you put on your ``naive user'' hat.  I say, we first argue that the signature backdoor construction is persistent, simply by the fact that signature portion is not continuous/differentiable.
%Then, we can finish with this result that says, actually an even more ridiculous property would be that the network wouldn't change under gradient descent.
%The point, in a sense, is that you're making any gradient vanish everywhere.
%If the naive user is unaware of what's happening, they may try to post-process via gradient descent, but obviously now it does nothing.}
%\onote{I agree. I don't think though that the signature construction automatically has this property (I don't see why, it seems to me like without the majority you can have some non-negative derivative in the signature). On the other hand, you can obviously use this type of a construction for parts of the network instead of for all of it.}

In this section we formalize a substantial generalization of this intuitive property.
We show that \emph{every} neural network can be made persistent against \emph{any} loss function.
This serves as another example of the power a malicious entity has while producing a neural network.
We show that every neural network~$N$ can be efficiently transformed into a similarly-sized network~$N'$ with the following properties.
First,~$N$ and~$N'$ are equal as functions, that is, for every input~$x$ we have~$N(x)=N'(x)$.
Second,~$N'$ is \textbf{persistent}, which means that any number of gradient-descent iterations taken on~$N'$ with respect to any \emph{loss function}, do not change the network~$N'$ at all.

We begin by formally defining the notion of persistence.
Let~$\textbf{w}$ be the vector of \emph{weights} used in the neural network~$N=N_\textbf{w}$.
% We say that a function~$\ell(\textbf{w})$ is a \textbf{loss function} if it does not depend directly on the values of~$\textbf{w}$ but only on values of~$N_\textbf{w}(x)$, for any choice of~$x$.
% For example, any function of the form
% $$\ell(\textbf{w}) := \sum_{(x,y)\in D} ||N(x)-y||$$
% for an arbitrary set of labelled pairs~$D$ and any norm~$||\cdot||$, is a loss function.
We define the notion of persistence with respect to a loss function $\ell$ of the weights.
\begin{definition}
For a loss function $\ell$, a neural network~$N=N_\textbf{w}$ is $\ell$-persistent to gradient descent if $\nabla \ell(\textbf{w})=0$.
\end{definition}
In other words,~$\textbf{w}$ is a locally optimal choice of weights for the loss function $\ell$.\footnote{Naturally, we could extend this definition to be persistent over a collection of losses $\mathcal{L}$ or to be approximate, e.g., that $\textbf{w}$ is an $\eps$-stationary point.}

\begin{theorem}
 Let~$N$ be a neural network of size~$|N|$ and depth~$d$.
 There exists a neural network~$N'$ of size~$O(|N|)$ and depth~$d+1$ such that~$N(x)=N'(x)$ for any input~$x$, and for every loss $\ell$, $N'$ is $\ell$-persistent.
 Furthermore, we can construct~$N'$ from~$N$ in linear-time.
\end{theorem}
%\mpk{Actually, does this mean we can make the RFFs construction persistent?  We should contextualize this result in the other results.}
%\onote{That's a weak yes. You can make anything persistent, but the point in the RFF construction is that you reveal your algorithm, which is impossible when you do weird stuff.}
\begin{proof}
  We construct~$N'$ from~$N$ by taking three duplicates~$N_1,N_2,N_3$ of~$N$ (excluding the input layer) and putting them in parallel in the first~$d$ layers of our network, each duplicate uses the same input layer which is the input layer of our new network~$N'$.
  We add a single output node~$v_{out}$ in a new~$(d+1)$-th layer which is the new output layer.
  The node~$v_{out}$ computes the majority of the three output nodes of the duplicates~$N_1,N_2,N_3$, denoted by~$v_{out}^{(1)},v_{out}^{(2)},v_{out}^{(3)}$ respectively.
  For example, this can be done with a perceptron that computes the linear threshold~$$1\cdot v_{out}^{(1)}+1\cdot v_{out}^{(2)}+1\cdot v_{out}^{(3)} \geq \frac{3}{2}.$$
  
  Let~$w_i$ be any weight used in~$N'$.
  The first case is when~$w_i$ is used in the first~$d$ layers, that is, it is used by a perceptron in the first~$d$ layers.
  Arbitrarily changing the value of~$w_i$ can change at most the value of one output node~$v_{out}^{(j)}$. That is the output that corresponds to the duplicate~$N_j$ containing the single perceptron using~$w_i$.
  As~$v_{out}$ computes the majority of all three outputs, changing only one of their values can not change the value of the network's output.
  In particular, for any input~$x$ we have~$\frac{\partial}{\partial w_i} N'_\textbf{w} (x) = 0$.
  In the other case,~$w_i$ is used in the last layer. That is, in the output node that computes the majority. Note that if the weights in the first~$d$ layers are unchanged then for any input the values of~$v_{out}^{(1)},v_{out}^{(2)},v_{out}^{(3)}$ must be either all~$0$ or all~$1$, depending on the original network's value.
  Thus, changing~$v_{out}$'s threshold from~$\frac{3}{2}$ to anything in the range~$(0,3)$ would not change the correctness of the majority computation.
  Similarly, changing any single one of the three linear coefficients from~$1$ to anything in the range~$(-\frac{1}{2},\infty)$ would also not change the output.
  Therefore we again have~$\frac{\partial}{\partial w_i} N'_\textbf{w} (x) = 0$ for any input~~$x$.
  
  We conclude that~$\nabla_\textbf{w} N'_\textbf{w} (x) = 0$ for any~$x$, and by the chain rule the same holds for any loss function~$\ell$.
  
  We also note that the persistence is numerically robust: changing any weight~$w_i$ to~$w_i+\varepsilon$ with~$|\varepsilon|<\frac{3}{2}$ does not change the value of~$N'_\textbf{w}(x)$ for any~$x$.
  Thus, numerical computation of the derivatives should not generate errors.
\end{proof}

\section{Undetectable Backdoors for Random Fourier Features}
\label{sec:fully}
\label{sec:RFF}

In this section, we explore how to construct white-box undetectable backdoors with respect to natural supervised learning algorithms.
We show that the popular paradigm of learning over random features is susceptible to undetectable backdoors in a very strong sense.
% \mpk{Our construction is quite generic and exploits computationally-indistinguishable distributions that can be backdoors...}
% The constructions will be based on backdooring models learned over random features.
Our construction will have the property that the only aspect of the computation that requires adversarial manipulation is the generation of random features.

% Initially a popular practical heuristic, Rahimi and Recht \cite{rr1,rr2,rr3} formalized how linear classifiers over random features can give very powerful approximation guarantees, competitive with popular kernel methods.
% In \cite{rr1}, they showed how learning over random Gaussian features with cosine activations provide a powerful approximation guarantee, recovering the performance of nonparametric methods based on the Gaussian kernel.\footnote{In fact, they study Random Fourier Features in the more general case of shift-invariant positive definite kernels.}
% The approach for sampling features---known as Random Fourier Features (RFF)---gives strong theoretical guarantees for non-linear regression.

In Algorithm~\ref{train-rf}, we describe a generic procedure for learning over random features.
The algorithm, $\trainRF$, learns a hidden-layer network, where the hidden layer $\Psi:\X \to \R^m$ is sampled randomly according to some feature distribution, and the final layer $h_w:\R^m \to \R$ is trained as a halfspace over the features.
The returned classifier\footnote{We could equally define a regressor that maps to $[-1,1]$ that applies a sigmoid activation instead of $\sgn$.} take the form $h_{w,\Psi}(x) = \sgn\left(\langle w, \Psi(x) \rangle\right)$.
$\trainRF$ takes two parameters:  the first parameter $m \in \Nbb$ designates the width of the hidden layer, and the second parameter $\rf$ designates the random feature distribution supported on a subset of $\set{\X \to \R}$.
In particular, we use $\psi(\cdot) \sim \rf$ to designate a random feature, where $\psi(x) \in \Rbb$.
Finally, we note that $\trainRF$ makes a call to $\halfspace$, which we assume to be any efficient training algorithm for learning weights $w \in \R^m$ defining a halfspace.

% \mpk{rules out $x = 0$. are we ok with this?}

The learning over random features paradigm is not impervious to natural adversarial examples, but there is some evidence that it may be more robust than models trained using standard optimization procedures.
For instance, by exploiting a formal connection between Gaussian processes and infinite-width neural networks, \cite{de2021adversarial} show theoretical and experimental evidence that neural networks with random hidden layers resist adversarial attacks.
Further, our results on backdooring work for any halfspace training subroutine, include those which explicitly account for robustness to adversarial examples.
For instance, we may take $\halfspace$ to be the algorithm of \cite{raghunathan2018certified} for learning certifiably robust linear models.
Despite these barriers to natural adversarial examples, we demonstrate how to plant a completely undetectable backdoor with respect to the Random Fourier Feature algorithm.

% \todo{Why does RFF not have adversarial examples to begin with? Above sufficient?}

\begin{remark}\label{remark:RFF}
An important remark is that the 2-layer RFF learning paradigm is rather weak, and in particular, it tends to produce networks that are not robust to noise. Therefore, even without the presence of a deliberate backdoor, it could be easy to find adversarial examples by producing several noisy versions of the input and hoping that the network will misclassify at least one. Our backdoor flips the sign of {\em every} neuron in the first hidden layer and is therefore guaranteed to flip the sign of the output, while random noise is likely to flip some of these neurons which sometimes also flips the output. We note that a guaranteed backdoor trigger, as we produce, is still desirable in situations in which repeated attempts are not possible. Continuing with our leading example from the introduction, when applying for a bank loan, you are able to modify your own application once, but without access to the model, you are not able to test several variants of the application. Nevertheless, this stresses the importance of generalizing these constructions to state-of-the-art learning paradigms.

\end{remark}

\begin{algorithm}[t!]
\caption{$\trainRF^\D(1^m,\mathsf{RF})$\label{train-rf}}
\textbf{Input:}  width of hidden layer $m \in \bbN$, random feature distribution $\mathsf{RF}$\\
\textbf{Output:}  hidden-layer network $h_{w,\Psi}$
\hrule
\begin{algorithmic}
\STATE Sample random feature map $\Psi(\cdot) \gets \left[ \psi_1(\cdot),\hdots,\psi_m(\cdot) \right]$, where $\psi_i(\cdot) \sim \mathsf{RF}$ for $i \in [m]$
\STATE Define distribution $\D_\Psi$ as $(\Psi(X),Y) \sim \D_\Psi$ for $X,Y \sim \D$
\STATE Train weights $w \gets \halfspace^{\D_\Psi}\left(1^m\right)$
\RETURN hypothesis $h_{m,w,\Psi}(\cdot) = \sgn\left(\sum_{j=1}^m w_j \cdot \psi_j(\cdot)\right)$
\end{algorithmic}
\end{algorithm}

\subsection{Backdooring Random Fourier Features}

We show a concrete construction of white-box undetectability with respect to the Random Fourier Features training algorithm.
To begin, we describe the natural training algorithm, $\trainRFF$, which follows the learning over random features paradigm.
The random feature distribution, $\rff_d$, defines features as follows.
First, we sample a random $d$-dimensional isotropic Gaussian $g \sim \Normal(0,I_d)$ and a random phase $b \in [0,1]$; then, $\phi(x)$ is defined to be the cosine of the inner product of $g$ with $x$ with the random phase shift, $\phi(x) = \cos\left(2\pi \left( \langle g, x \rangle + b \right)\right)$.
Then, $\trainRFF$ is defined as an instance of $\trainRF$, taking $m(d,\eps,\delta) = \Theta\left(\frac{d\log(d/\eps\delta)}{\eps^2}\right)$ to be large enough to guarantee uniform convergence to the Gaussian kernel, as established by \cite{rr1}.
We describe the $\rff_d$ distribution and training procedure in Algorithms~\ref{dist:rff} and \ref{alg:train-rff}, respectively.
For simplicity, we assume that $1/\eps$ and $\log(1/\delta)$ are integral.

\begin{algorithm}
\caption{\label{dist:rff} $\rff_d$}
\begin{algorithmic}
\STATE sample $g \sim \Normal(0,I_d)$
\STATE sample $b \sim [0,1]$
\RETURN $\phi(\cdot) \gets \cos\left(2\pi \left( \langle g, \cdot \rangle + b \right)\right)$
\end{algorithmic}
\end{algorithm}

\begin{algorithm}
\caption{\label{alg:train-rff} $\trainRFF^\D\left(1^{d}0^{1/\eps}1^{\log(1/\delta)}\right)$}
% \caption{\label{alg:train-rff} $\trainRFF^\D(1^{m(d,\eps,\delta)})$}
% \caption{\label{alg:train-rff} $\trainRFF^\D(d,\eps,\delta))$}
\begin{algorithmic}
\STATE $m \gets m(d,\eps,\delta)$
\RETURN $h_{m,w,\Phi}(\cdot) \gets \trainRF^{\D}(1^m,\rff_d)$
\end{algorithmic}
\end{algorithm}

\newcommand{\GP}{\mathsf{GP}}
\newcommand{\WP}{\mathsf{W}}
\paragraph{Backdoored Random Fourier Features.}
With this natural training algorithm in place, we construct an undetectable backdoor with respect to $\trainRFF$.
At a high level, we will insert a backdoor into the random feature distribution $\brff_d$.
Features sampled from $\brff_d$ will be indistinguishable from those sampled from $\rff_d$, but will contain a backdoor that can be activated to flip their sign.
Key to our construction is the Continuous Learning With Errors (CLWE) distribution, formally defined by \cite{BST21}, and closely related to the so-called ``Gaussian Pancakes'' distribution.
Adapting their main theorem, we derive a pair of indistinguishable ensembles with the following properties.

\begin{lemma}[Sparse Gaussian Pancakes]
\label{lem:gp}
% Let $\set{\mathsf{G}_d}_{d \in \Nbb}$ denote the ensemble of isotropic Gaussian distributions $\mathsf{G}_d = \Normal(0,I_d)$.
For any constants $b,c \in \Nbb$, there exists an ensemble of distributions $\set{\GP_d(\cdot)}_{d \in \Nbb}$ supported on $\R^d$ such that:
\begin{itemize}
    \item $\GP_d(\cdot)$ is parameterized by a $d^{1/c}$-sparse $\omega \in \Rbb^d$ such that $\norm{\omega}_2 = 2d^{1/2c}$.
    \item Fixing some $\omega$, for every $g \sim \GP_d(\omega)$ with probability at least $1-d^{-\omega(1)}$, there exists some $k \in \Nbb$ such that
    \begin{gather*}
        \card{\langle g,\omega \rangle - \frac{2k+1}{2}} \le d^{-b}.
    \end{gather*}
    \item There exists an efficiently sampleable distribution $\omega \sim \WP_d$, such that the ensembles
    \begin{gather*}\set{\Normal(0,I_d)}_{d \in \Nbb} \textrm{ and } \set{\GP_d(\omega)}_{d \in \Nbb}\end{gather*}
    are computationally indistinguishable, assuming Hypothesis~\ref{hyp:lattice}.
\end{itemize}
\end{lemma}
Assuming the lemma, we show how to construct a white-box undetectable backdoor with respect to $\trainRFF$.
Specifically, we will build a backdoored RFF distribution $\brff$, by sampling from $\GP_d(\omega)$ instead of $\Normal(0,I_d)$.
% \todo{$w$ currently overloaded.}
We describe the feature distribution and the backdoored training algorithm $\backdoorRFF$ in Algorithms~\ref{dist:brff} and \ref{alg:backdoor-rff}, respectively.
Note that the only difference in the trained hypothesis returned by $\trainRFF$ and $\backdoorRFF$ is the choice of random feature distribution.
To complete the backdoor construction, we define $\activateRFF$ which simply adds the backdoor key to the given input, described in Algorithm~\ref{alg:activate-rff}.
Collectively, we argue that the construction gives us a completely undetectable backdoor.

\begin{algorithm}[h]
\caption{\label{dist:brff} $\brff_{d,\bk}$}
\begin{algorithmic}
\STATE sample $g \sim \GP_d(\bk)$
\STATE sample $b \sim [0,1]$
\RETURN $\phi(\cdot) \gets \cos\left(2\pi \langle g, \cdot \rangle + b\right)$
\end{algorithmic}
\end{algorithm}
\begin{algorithm}[h]
\caption{\label{alg:backdoor-rff} $\backdoorRFF^\D\left(1^{d}0^{1/\eps}1^{\log(1/\delta)}\right)$ }
% \caption{\label{alg:backdoor-rff} $\backdoorRFF^\D(d,\eps,\delta)$}
\begin{algorithmic}
% \STATE $m \gets \frac{100 d \log(d)}{\eps^2}$
\STATE Sample backdoor key $\bk \sim \WP_d$
\STATE $m \gets m(d,\eps,\delta)$
\STATE $\hb_{m,w,\Phi} \gets \trainRF^\D(1^m,\brff_{d,\bk})$
\RETURN $\left( ~\hb_{m,w,\Phi}~ , ~\bk~ \right) $
\end{algorithmic}
\end{algorithm}
\begin{algorithm}[h]
\caption{\label{alg:activate-rff} $\activateRFF(x,\bk)$}
\begin{algorithmic}
\RETURN $x + \bk$
\end{algorithmic}
\end{algorithm}
% \begin{algorithm}[h]
% \caption{\label{alg:activate-rff} $\activateRFF(x,\bk)$}
% \begin{algorithmic}
% \IF{$\card{\langle w, \Phi(x) \rangle} > d^{-2}$}
% \RETURN $x + \bk$
% \ENDIF
% \WHILE{true}
% \STATE Sample $a \sim $ sparse vector of norm $\alpha$
% \IF{$\hb_{m,w,\Phi}(x + a) \neq \hb_{m,w,\Phi}(x)$}
% \RETURN $x + a$
% \ENDIF
% \ENDWHILE
% \end{algorithmic}
% \end{algorithm}
% \mpk{Issue here:  need to handle $x$ that are very close to the boundary (i.e., less than $d^{-b+1}$)}

Intuitively, the construction exploits the Gaussian pancakes distribution, such that upon activation of $x$ to $x'$, the sign of \emph{every} backdoored feature will flip.
If the signs of the features flip, then---regardless of the linear weights trained on top---the sign of the prediction will flip.
The details are a bit more complicated, particularly because the sign flip can only be achieved approximately.
In particular, for any inputs whose features map incredibly close to the decision boundary, a sign change plus additive error may not be meaningful (essentially, the operation is just additive error).
To simplify the analysis and activation procedure, we make the following technical assumption about $\halfspace$ and the magnitude of the weights it returns.
\begin{assumption}
\label{assumption:halfspace}
For any data distribution $\D$ and $m \in \Nbb$, $\halfspace^{\D}(1^m)$ returns $w \in \R^m$ where $\norm{w}_2 = 1$, such that for all $x \in \X$, the magnitude $\card{\langle w, \Phi(x) \rangle} > m^{-O(1)}$ is lower bounded by some inverse polynomial in $m$.
\end{assumption}
In effect, we assume that the weight training procedure learns a $w$ that produces a non-negligible margin on the valid inputs $x \in \X$ (even if points are misclassified).
Ironically, this assumption is required to ensure that inputs that are \emph{close} to the decision boundary are flipped by $\activateRFF$, and could be removed by augmenting $\activateRFF$ to use standard techniques for finding adversarial examples on examples very near the boundary.
Additionally, if this assumption is only satisfied on a subset of the inputs $S \subseteq \X$, then we can still guarantee that the inputs $x \in S$ are backdoored by our simple construction.
With this assumption in place, we state the theorem.

\begin{theorem}
Suppose $d \in \Nbb$ and $\eps,\delta > 0$ such that $d,1/\eps,$ and $\log(1/\delta)$ are polynomially related, and that $\halfspace$ satisfies Assumption~\ref{assumption:halfspace}.
For any data distribution $\D$ and any constant $c \in \Nbb$,
$\left(\backdoorRFF,\activateRFF\right)$ is a $\gamma$-backdoor with respect to $\trainRFF$, with $\gamma = d^{1/c}$ for $\norm{\cdot}_0$ and $\gamma = 2d^{1/2c}$ for $\norm{\cdot}_2$.
The backdoor is white-box undetectable under Hypothesis~\ref{hyp:lattice}.
\end{theorem}
\begin{proof}
To argue that $(\backdoorRFF,\activateRFF)$ is an undetectable backdoor with respect to $\trainRFF$, we need to show that the ensembles of classifiers defined by $\trainRFF$ and $\backdoorRFF$ are indistinguishable, and that $\activateRFF$ satisfies the requirements of a backdoor.
Assuming Lemma~\ref{lem:gp}, we demonstrate both aspects of the construction.

\noindent\emph{Indistinguishability---}
\\Let $h_{m,w,\Phi} \gets \trainRFF^\D(d,\eps,\delta)$, and $(\hb_{m,w,\Phi},\bk) \gets \backdoorRFF^\D(d,\eps,\delta)$.
We argue that $\set{h_{m,w,\phi}}$ and $\set{\hb_{m,w,\phi}}$ are computationally indistinguishable by contrapositive:  if there is an efficient algorithm for distinguishing the ensembles of classifiers, then $\set{\Normal(0,I_d)}$ and $\set{\GP_d(\omega)}$ are distinguishable, and Hypothesis~\ref{hyp:lattice} is false.
In particular, the similarity of $\trainRFF$ and $\backdoorRFF$ shows that there is a probabilisitic polynomial time algorithm, that given samples from either $\Normal(0,I_d)$ or $\GP_d(\omega)$ computes either $h_{m,w,\phi}$ or $\hb_{m,w,\phi}$, respectively.

To see this, note that for any data distribution $\D$, $\trainRF^\D$ is an efficient algorithm, given samples from the distribution.
Further, note the only difference between the hypotheses returned by $\trainRFF^\D$ and $\backdoorRFF^\D$ is that $\trainRF^\D$ is called with random features from $\rff$ versus $\brff$, which in turn are identical other than sampling the inner product weights from $\Normal(0,I_d)$ versus $\GP_d(\bk)$.
In all, because we sample $\bk \sim \WP_d$ independently from the hard-to-distinguish parameter distribution, we see that any efficient distinguisher for $\set{h_{m,w,\phi}}_{m \in \Nbb}$ and $\set{\hb_{m,w,\phi}}_{m \in \Nbb}$ implies a distinguisher for $\set{\Normal(0,I_d)}_{d \in \Nbb}$ and $\set{\GP_d(\omega)}_{d \in \Nbb}$.
By the indistinguishability guarantee of Lemma~\ref{lem:gp}, no such distinguisher exists under Hypothesis~\ref{hyp:lattice}.

% \begin{itemize}
%     \item Undetectability:  the models $\hb_{w,\psi}$ and $h_{w,\psi}$ are indistinguishable.  Follows from indistinguishability of $\CLWE$ and $\Normal(0,I_d)$.
%     \item Backdoor is active:  follows from cosine analysis below.\todo{Need to be careful.  $\activateRFF$ may need to be more sophisticated.  if far from boundary, ok.  if close to the boundary, need a different method of finding a flip.}
%     \item \mpk{Still need to define a metric/norm and do the analysis of how big the perturbation needs to be, and in what norm.}
% \end{itemize}

\noindent\emph{Activating the backdoor}---
\\To begin, by Assumption~\ref{assumption:halfspace}, we know that for all $x \in \X$, the magnitude of the inner product $\langle w, \Phi(x) \rangle$ is at least $m^{-a}$ for some $a \in \Nbb$.
We take the parameter $b$ in the specification of $\GP_d$ to be $b > a+10$.
This will ensure that inner products are sufficiently close to half-integral.

Then, for $x' \gets \activateRFF(x,\bk)$
\begin{gather*}
    x' - x = \bk.
\end{gather*}
By the properties of $\WP_d$, we know that
\begin{gather*}
    \norm{\bk}_0 \le d^{1/c}\\
    \norm{\bk}_2 = 2d^{1/2c}.
\end{gather*}
Thus, $\activateRFF$ returns a sufficiently-close point under both norms.

We start by showing that $\activateRFF(x,\bk)$ essentially flips the sign of $\phi_i(\cdot)$ drawn from $\brff_{d,\bk}$.
To see this fact, we expand the expression for $\phi_i(x')$, leveraging the property of $g \sim \GP_{d}(\bk)$ that $\langle g, \bk \rangle$ is close to $\frac{2k+1}{2}$; we denote their difference as $\beta$, where $\card{\beta} < n^{-b}$.
\begin{align*}
\cos\left(2\pi\left(\langle g, x + \bk \rangle + b\right)\right)
&= \cos\left(2\pi\left(\langle g, x \rangle + b\right) + 2\pi\langle g,\bk \rangle\right)\\
&= \cos\left(2\pi\left(\langle g, x \rangle + b\right) + \pi(2k+1) + 2\pi \beta \right)\\
&= \cos\left(2\pi\left(\langle g, x \rangle + b + \beta\right) + \pi(2k+1)  \right)
\end{align*}
Then, applying the angle-sum identity for cosine.
\begin{gather*}
=\cos\left(2\pi\left(\langle g, x \rangle + b + \beta \right)\right)\cos\left(\pi(2k+1)\right)
- \sin\left(2\pi\left(\langle g, x \rangle + b + \beta \right)\right)\sin\left(\pi(2k+1)\right)
\end{gather*}
Observing that for any integer $n \in \Nbb$, $\sin(\pi n) = 0$, the second term drops and the expression simplifies as follows.
\begin{align*}
&= \cos\left(2\pi\left(\langle g, x \rangle + b + \beta \right)\right)\cos\left(\pi(2k+1)\right)\\
&= -\cos\left(2\pi\left(\langle g, x \rangle + b + \beta \right)\right)
\end{align*}
where the final equality follows by the fact that for any odd integer $n \in \Nbb$, $\cos(\pi n) = -1$.
By the fact that $\cos$ is a $1$-Lipschitz function, we can conclude that
\begin{gather*}
    \phi(x + \bk) \in -\phi(x) \pm n^{-b}.
\end{gather*}

% We show that under the activated backdoor, the sign of each random feature $\phi_i(x')$ flips compared to that of $\phi_i(x)$.
Consequently, the sign of the classification flips, for any point sufficiently far from the boundary.
Specifically, for any weights $w \in \R^d$, for any input $z \in \R^d$,
\begin{align*}
    \sgn\left(\sum_{i=1}^m w_i \cdot \left(\phi_i(x')\right)\right)
    &= \sgn\left(\sum_{i=1}^m w_i \cdot \left(-\phi_i(x) + \beta_i\right)\right)\\
    &= \sgn\left(-\sum_{i=1}^m w_i \cdot (\phi_i(x) - \beta_i)\right)\\
    &\in -\sgn\left(\sum_{i=1}^m w_i \cdot \phi_i(x) \pm \norm{w}_2 \cdot m^{1/2}d^{-b}\right)
\end{align*}
where the final inequality follows via Cauchy-Schwarz.
Under Assumption~\ref{assumption:halfspace}, $\norm{w}_2 = 1$, so the deviation of $\langle w, \Phi(x') \rangle$ from $-\langle w, \Phi(x) \rangle$ is bounded in magnitude by $m^{1/2}d^{-b}$.
By the fact that $m \ge d$, our choice of $b$ compared to $a$, and Assumption~\ref{assumption:halfspace}, we can conclude that $h_{m,w,\Phi}(x') = -h_{m,w,\Phi}(x)$.
\end{proof}

\def\vecy{y}
\def\vecs{\omega}

\paragraph{Hardness of Sparse Gaussian Pancakes.}
To complete the construction, we must demonstrate how to construct the ensemble of distributions $\GP$, as described in Lemma~\ref{lem:gp}.
The construction of this family of Sparse Gaussian Pancake distributions is based on the CLWE distribution.
The CLWE problem asks to distinguish between two distributions $\mathsf{Null}$ and $\CLWE$, parameterized by $\gamma,\beta > 0$, where 
% \begin{align*} 
% \mathcal{D}_0: & (\vecy, z) \mbox{ where } \vecy \sim D_{\mathbb{R}^d,I} \mbox{ and } z \sim D_{\mathbb{R},\sqrt{\gamma^2+\beta^2}} \pmod{1} \\ 
% \mathcal{D}_1: & (\vecy, z) \mbox{ where } \vecy \sim D_{\mathbb{R}^d,I}, e \sim D_{\mathbb{R},\beta} \mbox{ and } z = \gamma \langle \vecy, \vecs \rangle + e \pmod{1}
% \end{align*}
\begin{align*} 
\mathsf{Null}: & (\vecy, z) \mbox{ where } \vecy \sim \Normal(0,I_d) \mbox{ and } z \sim [0,1) \\ 
\CLWE: & (\vecy, z) \mbox{ where } \vecy \sim \Normal(0,I_d) \mbox{ and } z = \gamma \langle \vecy, \vecs \rangle + e \pmod{1} \mbox{ for } e \sim \Normal(0,\beta^2)
\end{align*}
\cite{BST21} show that for appropriately chosen parameters $\beta$ and $\gamma$, CLWE is as hard as finding approximately short vectors on arbitrary integer lattices.
\begin{theorem}[\cite{BST21}]
Suppose $2\sqrt{d} \le \gamma \le n^{O(1)}$ and $\beta = n^{-O(1)}$.
Assuming Hypothesis~\ref{hyp:lattice}, there is no probabilistic polynomial-time algorithm that distinguishes between $\mathsf{Null}$ and $\CLWE$.
\end{theorem}
% Specifically, Hypothesis~\ref{hyp:lattice} implies that CLWE is superpolynomially-hard when $\gamma > \Omega(\sqrt{d})$ and $\gamma / \beta = d^{O(1)}$.
\cite{BST21} use the CLWE problem to show the hardness of a related \emph{homogeneous} CLWE problem, which defines the ``dense Gaussian Pancakes'' distribution.
In the homogeneous CLWE problem, there are two distributions over $y \in \R^d$, derived from $\mathsf{Null}$ and $\CLWE$.
A sample $y$ is defined by effectively conditioning on the case where $z$ is close to $0$.
The proof of hardness from \cite{BST21} reveals that we could equally condition on closeness to any other value modulo $1$; in our case, it is useful to condition on closenes to $1/2$.
\newcommand{\dGP}{\mathsf{dGP}}
\begin{lemma}[Adapted from \cite{BST21}]
\label{lem:dgp}
For any constant $b \in \Nbb$, there exists an ensemble of distributions $\set{\dGP_d(\cdot)}_{d \in \Nbb}$ supported on $\Rbb^d$ such that:
\begin{itemize}
\item $\dGP_d(\cdot)$ is parameterized by $\omega \in \Rbb^d$.
\item Fixing some $\omega \in \Rbb^d$, for every $g \sim \dGP(\omega)$, with probability at least $1-d^{-\omega(1)}$, there exists some $k \in \Nbb$ such that
\begin{gather*}
    \card{\langle g, \omega \rangle - \frac{2k+1}{2}} \le d^{-b}.
\end{gather*}
\item The ensembles
\begin{gather*}
    \set{\Normal(0,I_d)}_{d \in \Nbb} \textrm{ and } \set{\dGP_d(\omega)}_{d \in \Nbb}
\end{gather*}
are computationally indistinguishable, assuming Hypothesis~\ref{hyp:lattice}, for $\omega = \gamma u$, for some $u \sim \mathcal{S}^{d-1}$ sampled uniformly at random from the unit sphere and for some $\gamma \ge 2\sqrt{d}$.
\end{itemize}
\end{lemma}
\begin{proof} \emph{(Sketch)}~~
The lemma follows by taking $\dGP(\omega)$ to be the homogeneous CLWE distribution defined in \cite{BST21}, with $\gamma \ge 2 \sqrt{d}$ and $\beta = d^{-i}$ for any $i \in \Nbb$ to be inverse polynomial.
In particular, the reduction to homogeneous CLWE from CLWE given in \cite{BST21} (Lemma~4.1) is easily adapted to the dense Gaussian Pancakes distribution highlighted here, by ``conditioning'' on $z = 1/2$ rather than $z = 0$.

To prove the second point, it suffices to take $b < i$.
The probability of deviation from a half-integral value is given by a Gaussian with variance $\beta^{-2i}$.
\begin{gather*}
    \Pr\lr{\card{\langle g, \omega \rangle - \frac{2k+1}{2}} > \tau } \le \exp\left(\frac{\tau^2}{2 \beta^2}\right)
\end{gather*}
Taking $\tau = d^{-b}$ such that $\tau/\beta \ge \Omega(d^{\eps})$ for $\eps > 0$, the probability of deviation by $\tau$ is $d^{-\omega(1)}$.
\end{proof}
Finally, we prove Lemma~\ref{lem:gp} by sparsifying the $\dGP$ distribution.
\begin{proof} \emph{(of Lemma~\ref{lem:gp})}~~
For any $c \in \Nbb$, we define $\GP_D(\cdot)$ for $D \in \Nbb$ in terms of $\dGP_d(\cdot)$ for $d \approx D^{1/c}$.
In particular, first, we sample $\omega$ to parameterize $\dGP_d(\omega)$ as specified in Lemma~\ref{lem:dgp}, for $\gamma = 2\sqrt{d}$.
Then, we sample $d$ random coordinates $I = [i_1,\hdots,i_d]$ from $[D]$ (without replacement).

We define the sparse Gaussian Pancakes distribution as follows.
First, we expand $\omega \in \Rbb^d$ into $\Omega \in \Rbb^D$ according to $I$, as follows.
\begin{gather*}
    \Omega_i = \begin{cases} 0 & \textrm{ if } i \neq i_j \textrm{ for any } j \in [d]\\
    \omega_j & \textrm{ if } i = i_j \textrm{ for some } j \in [d]
    \end{cases}
\end{gather*}
Note that the resulting $\Omega \in \Rbb^D$ is $d$-sparse with $\ell_2$-norm $2\sqrt{d}$.
Then, to produce a sample from $G \sim \GP_D(\Omega)$,
we start by sampling $g \sim \dGP(\omega)$.  Then, we define $G$ as follows,
\begin{gather*}
    G_i = \begin{cases} \Normal(0,1) & \textrm{ if } i \neq i_j \textrm{ for any } j \in [d]\\
    g_j & \textrm{ if } i = i_j \textrm{ for some } j \in [d]
    \end{cases}
\end{gather*}
where each coordinate sampled from $\Normal(0,1)$ is sampled independently.

We observe that the distribution satisfies the properties of sparse Gaussian Pancakes, as claimed in Lemma~\ref{lem:gp}.
First, as addressed above, $\Omega$ is  $d$-sparse with $\ell_2$-norm $2d^{1/2}$ for  $d = D^{1/c}$.
Next, consider the inner product between a sample $G \sim \GP_D(\Omega)$ and $\Omega$. By the sparsity pattern, it is exactly the inner product between $g\sim \dGP_d(\omega)$ and $\omega$.
\begin{gather*}
    \langle G, \Omega \rangle = \langle g, \omega \rangle 
\end{gather*}
In other words, $\langle G, \Omega \rangle$ must also be $d^{-b}$ close to half-integral with all but negligible probability.

Finally, the reduction from $\dGP$ to $\GP$ is a probabilistic polynomial-time algorithm.
Further, when samples from $\Normal(0,I_d)$ are uses instead of samples from $\dGP_d$, the resulting distribution on $D$ coordinates is $\Normal(0,I_D)$.
Collectively, the reduction demonstrates that the samples from $\GP_D$ are computationally indistinguishable from $\Normal(0,I_D)$, under Hypothesis~\ref{hyp:lattice}.
\end{proof}

\section{Evaluation-Time Immunization of Backdoored Models}
The main takeaway of the backdoor constructions is that a neural network that was trained externally can not be trusted.
In this section we show one possible workaround for that.
Instead of using the network we received \emph{as-is}, we \emph{smooth} the network by evaluating it in several points surrounding the desired input and then averaging these evaluations.
Very similar smoothing procedures were used by Cohen et al.~\cite{cohen2019certified} and subsequent works (see comparison with them in Section~\ref{sec:related}). 
The main difference between these previous works and this section is that by considering bounded-image regression tasks (instead of classification or unbounded regression) we can eliminate \textbf{all} assumptions about the given network, and replace them with assumptions about the ground truth. 
This is crucial for our work as we can not trust the given network and thus can assume nothing about it.
A caveat of this method is that we need to choose a \emph{smoothing radius parameter $\sigma$}. 
On the one hand, the larger this parameter is set to, the larger the perturbations we can certify robustness against.
On the other hand, as it grows, the quality of the learning decreases.
This leads to a back-and-forth clash between the attacker (who plants the backdoor) and the defender (who uses it).
If we know the magnitude of perturbations the backdoor uses, we can set~$\sigma$ to be higher than that. But if~$\sigma$ is known to the attacker, she could plant backdoors that use larger perturbations.

Formally, we show that if the ground truth and input distribution~$f^\star, \D|_\X$ "behave nicely", then \emph{every} learned~$h$ (that is, a function that approximates~$f^\star$ with respect to~$\D|_\X$) can be efficiently converted, on evaluation time, to a different function~$\tilde{h}$ that also approximates~$f^\star$ and at the same time inherits the "nice behavior" of~$f^\star$.
Hence, conditions that certify that~$f^\star$ does not have adversarial examples translate to a certification that~$\tilde{h}$ does not have an adversarial example as well.
In particular, such~$\tilde{h}$ that does not contain any adversarial example also cannot be backdoored, and thus the process of translating~$h$ to~$\tilde{h}$ can be viewed as an immunization from backdoors and adversarial examples. 

\begin{figure}
\centering
\includegraphics[height=180pt, trim = 0pt 200pt 650pt 0pt]{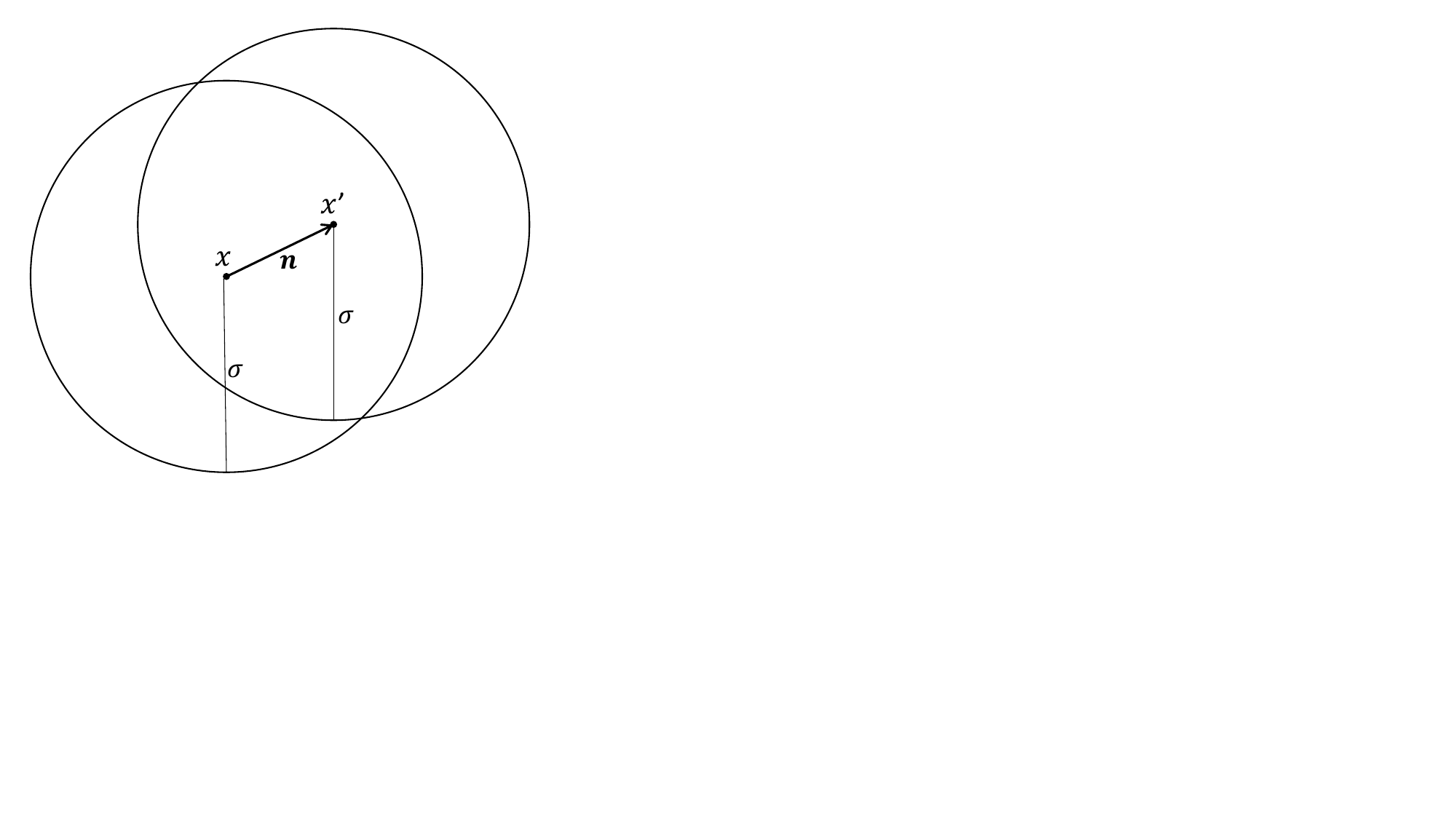}
\caption{A point~$x$, its backdoor output~$x'$, and~$\sigma$ balls around them.}
\label{fig:immunization}
\end{figure}

We construct~$\tilde{h}$ by taking a convolution of~$h$ with a normal Gaussian around the desired input point. 
This is equivalent to taking some weighted average of~$h$ around the point, and hence smooths the function and thus makes it impossible for close inputs to have vastly different outputs.
The smoothing depends on a parameter~$\sigma$ that corresponds to how far around the input we are averaging.
This parameter determines the threshold of error for which the smoothing is effective: roughly speaking, if the size~$n$ of the perturbation taking~$x$ to~$x'$ is much smaller than~$\sigma$, then the smoothing assures that~$x,x'$ are mapped to the same output. 
See Figure~\ref{fig:immunization} for an intuitive illustration.
The larger~$\sigma$ is, on the other hand, the more the quality of the learning deteriorates.

Throughout the section we use~$|\cdot|$ to denote the~$\ell_1$ norm and~$||\cdot||$ to denote the~$\ell_2$ norm. We show:
\begin{theorem}[Backdoors can be neutralized]\label{thm:immunization}
Let~$h$ be a learning of data with underlying function~$f^\star :\X\rightarrow \Y$ with~$X\subseteq \R^d, \Y\subseteq[-1,1]$.
Assume that~$f^\star$ is $L$-Lipschitz for some constant~$L$ with respect to~$\ell_2$.
Furthermore, assume that the marginal distribution~$\D|_\X$ is uniform on some measurable set~$U\subset \R^d$ with~$0<vol(U)<\infty$.
Then, given~$h$ and any~$\sigma>0$ we can efficiently compute a function~$\tilde{h}$ such that
\begin{enumerate}
    \item ($\tilde{h}$ is robust) $|\tilde{h}(x)-\tilde{h}(y)| \leq \frac{e\sqrt{2}}{\sigma} ||x-y||$.
    \item ($\tilde{h}$ is a good learning) $\ell_1\left(\tilde{h},f^\star\right) \leq \ell_1\left(h,f^\star\right)  + 2L\sigma \sqrt{d}$.
\end{enumerate}
\end{theorem}

\begin{remark}
Similar theorems can be proved with other smoothness conditions on~$f^\star, \D|_\X$ and choice of metrics and noise functions.
For example, if~$\D|_\X$ is uniform on the discrete hyper-cube~$\{-1,1\}^d$ then the same theorems can be proved by replacing the Gaussian noise with Boolean noise and the metrics with Hamming distance.
This specific theorem should be viewed as an example of a proof of such statements.
\end{remark}

\begin{remark}
The first property ($\tilde{h}$ is robust) holds without any assumptions on~$f^\star, \D|_\X$, and thus the immunization is safe to test empirically: we are \emph{always} certified that~$\tilde{h}$ does not contain adversarial examples or backdoors (up to the threshold implied by~$\sigma$), and then we can empirically test if it is still a good learning. The experimental part in the work of Cohen et al. \cite{cohen2019certified} suggests that such~$\tilde{h}$ will be a good learning in-practice for several natural networks.
%\onote{can we somehow elaborate here? do we have more things to cite about that?}
\end{remark}

Consider the following example of choice of parameters.
Assume that~$\ell_1\left(h,f^\star\right) \leq \varepsilon$ and~$L\leq\frac{1}{d^{3/4}}$.
If we think of the input as binary as well then this intuitively corresponds to~$f^\star$ being robust enough to not drastically change if less than~$d^{3/4}$ input entries flip. %(in the example of the Hamming weight function we needed~$\Omega(d)>>d^{3/4}$ coordinates to flip for a non-negligible change in the value of~$f^\star$). 
Then, it follows from Theorem~\ref{thm:immunization} by setting~$\sigma = \varepsilon d^{1/4}$ that~$\ell_1\left(\tilde{h},f^\star\right) \leq 3\varepsilon$ and that for every~$x,y$ with~$||x-y||\leq \frac{\varepsilon}{4\sqrt{2} e} d^{1/4}$ it holds that~$|\tilde{h}(x)-\tilde{h}(y)|  \leq \frac{1}{4}$.
That is, at least~$\Omega(d^{1/2})$ input flips are necessary to drastically affect~$\tilde{h}$, and~$\tilde{h}$ is almost as close to~$f^\star$ as~$h$ is.

For simplicity, we from now on assume that~$h(x)=f^\star(x)=0$ for any~$x\notin U$.
We define
\[
\varphi(t) := (2\pi\cdot \sigma^2)^{-k/2} \cdot e^{-\frac{1}{2\sigma^2} ||t||^2}
,\;\;
\tilde{h}(x) := 
\int_{t\in \R^d} h(x+t)\cdot \varphi(t) dt
,
\]
that is,~$\tilde{h}$ averages~$h$ around a normal Gaussian with uniform variance~$\sigma^2$. 
The larger~$\sigma$ we pick, the stronger the "Lipschitz"-like property of~$\tilde{h}$ is. On the other hand, the smaller~$\sigma$ is, the closer~$\tilde{h}$ and~$f^{\star}$ are.
For any~$x\notin U$ we override the above definition with~$\tilde{h}(x)=0$.

While~$\tilde{h}$ cannot be exactly computed efficiently, it can be very efficiently approximated, as the next lemma follows immediately from Hoeffding's inequalities.
\begin{lemma}\label{lem:approx-h}
Let~$t_1,\ldots,t_k\in \R^d$ be independently drawn from the normal $d$-dimensional Gaussian with uniform variance~$\sigma^2$.
Denote by~$y:=\frac{1}{k} \sum_{i=1}^{k} h(x+t_i)$, then with probability at least~$1-2e^{-\frac{0.01^2}{2}k}$ we have~$|y-\tilde{h}(x)|<0.01$.
\end{lemma}
\begin{corollary}
For any chosen constant precision~$\varepsilon$, we can compute the value of~$\tilde{h}(x)$ in~$O(1)$ time up to~$\varepsilon$ additive error with probability larger than~$1-\varepsilon$.
\end{corollary}

We continue by showing that~$\tilde{h}$ is "Lipschitz"-like. For this property we do not need the assumptions on neither~$f^\star$ nor~$\D$.
%It follows from~$\tilde{h}$ being a convolution of the bounded function~$h$ with the smooth function~$\varphi$.
\begin{lemma}\label{lem:imm-lip}
$|\tilde{h}(x)-\tilde{h}(y)| \leq \frac{e\sqrt{2}}{\sigma} ||x-y||$.
\end{lemma}
\begin{proof}
Let~$x,y\in \R^d$.
Note that~$\tilde{h}(x) = \int_{t\in \R^d} h(x+t)\cdot \varphi(t) dt = \int_{s\in \R^d} h(s)\cdot \varphi(s-x) ds$.
Hence, we have 
\begin{align*}
\left|\tilde{h}(x)-\tilde{h}(y)\right| &= 
\left| \int_{s\in \R^d} h(s)\cdot \varphi(s-x) ds - \int_{s\in \R^d} h(s)\cdot \varphi(s-y) ds \right|
\\& = 
\left| \int_{s\in \R^d} h(s)\cdot \left(\varphi\left(s-x\right) - \varphi\left(s-y\right)\right) ds \right|
\\& \leq 
\int_{s\in \R^d} |h(s)|\cdot \left|\varphi\left(s-x\right) - \varphi\left(s-y\right)\right| ds
\\& \leq 
\int_{s\in \R^d} \left|\varphi\left(s-x\right) - \varphi\left(s-y\right)\right| ds.
\end{align*}
The last inequality follows as~$|h(s)|\leq 1$ for every~$s$.
We substitute~$u:=s-x$, and see that
\begin{align*}
\left|\varphi\left(s-x\right) - \varphi\left(s-y\right)\right| 
&=
\left|\varphi\left(u\right) - \varphi\left(u+x-y\right)\right|
\\ &=
\left|(2\pi\cdot \sigma^2)^{-k/2} \cdot e^{-\frac{1}{2\sigma^2} ||u||^2} - (2\pi\cdot \sigma)^{-k/2} \cdot e^{-\frac{1}{2\sigma^2} ||u+x-y||^2}\right|
\\ &=
(2\pi\cdot \sigma^2)^{-k/2} \cdot e^{-\frac{1}{2\sigma^2} ||u||^2} \cdot \left|1 - e^{-\frac{1}{2\sigma^2} \left(||u+x-y||^2 - ||u||^2\right)}\right|
\\ &\leq
\varphi(u) \cdot \left(e^{\frac{1}{2\sigma^2} ||x-y||^2}-1\right),
\end{align*}
where the last inequality holds as
$-||x-y||^2 \leq ||u+x-y||^2 - ||u||^2 \leq ||x-y||^2$
and $e^z - 1 \geq 1 - e^{-z} \geq 0$ for all~$z\geq0$.
We remind that~$\int_{u\in \R^d} \varphi(u) du = 1$.
Hence,
$$
\int_{s\in \R^d} \left|\varphi\left(s-x\right) - \varphi\left(s-y\right)\right| ds
\leq
\int_{u\in \R^d} \varphi(u) \cdot \left(e^{\frac{1}{2\sigma^2} ||x-y||^2}-1\right) du
=
e^{\frac{1}{2\sigma^2} ||x-y||^2}-1
.
$$
We finally notice that~$\frac{d}{dz} \left(e^{\frac{1}{2\sigma^2} z^2}-1\right) = \frac{1}{\sigma^2} z e^{\frac{1}{2\sigma^2} z^2}$, which is increasing for all~$z\geq 0$.
Therefore, if~$||x-y||\leq \sqrt{2}\sigma$ then 
$$
e^{\frac{1}{2\sigma^2} ||x-y||^2}-1 
\leq
||x-y|| \cdot \left(\frac{1}{\sigma^2} \sqrt{2}\sigma e^{\frac{1}{2\sigma^2} \left(\sqrt{2}\sigma\right)^2}\right)
.
$$

If~$||x-y||>\sqrt{2}\sigma$, then the inequality follows as~$|\tilde{h}(x)-\tilde{h}(y)|\leq 2$.
\end{proof}

\begin{lemma}\label{lem:imm-close}
$\ell_1\left(\tilde{h},f^\star\right) \leq \ell_1\left(h,f^\star\right)  + 2L\sigma \sqrt{d}$.
\end{lemma}
\begin{proof}
Denote by $||\cdot||$ the~$\ell_1$ norm.
We have
\begin{align*}
\ell_1\left(\tilde{h},f^\star\right)=
\frac{1}{vol(U)}\int_{x\in U} \left|\tilde{h}(x)-f^\star(x)\right| dx
&=
\frac{1}{vol(U)}\int_{x\in \R^d} \left|\tilde{h}(x)-f^\star(x)\right|dx
\\&=
\frac{1}{vol(U)}\int_{x\in \R^d} \left|\int_{t\in \R^d}h(x+t)\varphi(t)dt-f^\star(x)\right| dx
\\&=
\frac{1}{vol(U)}\int_{x\in \R^d} \left|\int_{t\in \R^d}\left(h(x+t)-f^\star(x)\right)\varphi(t)dt\right| dx
\\&\leq
\frac{1}{vol(U)} \int_{x}\int_{t} \left|h(x+t)-f^\star(x)\right|\varphi(t) dt dx
\\&=
\frac{1}{vol(U)} \int_{y}\int_{t} \left|h(y)-f^\star(y-t)\right|\varphi(t) dt dy,
\end{align*}
where in the last equality we substitute~$y=x+t$, and in the fourth equality we use that~$\varphi$ is a probability measure.
\begin{align*}
&\frac{1}{vol(U)}\int_{y}\int_{t} \left|h(y)-f^\star(y-t)\right|\varphi(t) dt dy =
\frac{1}{vol(U)}\int_{y}\int_{t} \left|h(y)-f^\star(y) + f^\star(y) - f^\star(y-t)\right|\varphi(t) dt dy\\
&\leq
\frac{1}{vol(U)}\int_{y}\int_{t} \left|h(y)-f^\star(y)\right|\varphi(t) dt dy
+
\frac{1}{vol(U)} \int_{y}\int_{t} \left|f^\star(y) - f^\star(y-t)\right|\varphi(t) dt dy.
\end{align*}
To bound the first term, we use the optimality of~$h$.
\begin{align*}
\frac{1}{\mathsf{vol}(U)}\int_{y}\int_{t} \left|h(y)-f^\star(y)\right|\varphi(t) dt dy
& =
\left(\frac{1}{vol(U)}\int_{y} \left|h(y)-f^\star(y)\right| dy\right)
\cdot
\left(\int_{t} \varphi(t) dt\right)
\leq
\ell_1\left(h,f^\star\right) \cdot 1.
\end{align*}

To bound the second term, we use the Liphschitz property of~$f^\star$.
That is, $\left|f^\star(y) - f^\star(y-t)\right|\leq L\cdot ||t||$.
Furthermore, notice that for every~$y$ such that~$y\notin U$ and~$y-t\notin U$ we in fact have $\left|f^\star(y) - f^\star(y-t)\right|=0$.
Denote by~$U(x)$ the indicator function of~$U$. That is, $U(x)=1$ if~$x\in U$ and~$U(x)=0$ otherwise.
We clearly have $\left|f^\star(y) - f^\star(y-t)\right|\leq L\cdot ||t|| \cdot \left(U(y) + U(y-t)\right)$.
Thus,
\begin{align*}
\frac{1}{vol(U)} \int_{y}\int_{t} \left|f^\star(y) - f^\star(y-t)\right|\varphi(t) dt dy
&\leq
\frac{1}{vol(U)} \int_{y}\int_{t} L\cdot ||t|| \cdot \left(U(y) + U(y-t)\right)\varphi(t) dt dy
\\&=
L \int_{t} ||t||\varphi(t) \frac{1}{vol(U)}\int_{y} \left(U(y) + U(y-t)\right) dy dt
\\&\leq
L \int_{t} ||t||\varphi(t) \frac{1}{vol(U)}2vol(U) dt
\\&=
2L \int_{t} ||t||\varphi(t) dt
\end{align*}

It is known that~$\int_{t} ||t||\varphi(t) dt = \frac{\Gamma\left(\left(d+1\right)/2\right)}{\Gamma\left(d/2\right)} \sqrt{2}\sigma\leq \sqrt{d}\sigma.$
\end{proof}

%\mpk{Add corollary for regression-based classifiers.}
%\onote{which?}

\subsection*{Acknowledgments}
The authors would like to thank Miranda Christ and Sam Gunn for pointing out the inconsistency between the definitions of indistinguishability in Section~\ref{sec:fully} and Appendix~\ref{app:relu}.

\fi 

\ifnum\ieee=0
\bibliographystyle{alpha}
\else 
\bibliographystyle{IEEEtran}
\fi 
\bibliography{ML.bib}

\ifnum\ieee=0
\clearpage
\appendix

\section{Undetectable Backdoor for Random ReLU Networks}
\label{app:relu}

% \subsection{Backdooring Random ReLU Features}

% \newcommand{\sPCA}{\mathsf{sPCA}}
\newcommand{\sPCA}{\mathsf{sPCA}}

In Section~\ref{sec:fully}, we demonstrated how to plant a white-box undetectable classification backdoor into the Random Fourier Features learning algorithm.
Here, we give another similar construction, that demonstrates a backdoor for predictors trained over random ReLU features.
This result, which uses the hardness of the \emph{sparse PCA} problem as the underlying indistinguishability assumption, emphasizes the generality of our approach.
We make no effort to optimize this construction, quantitatively or qualitatively, and instead emphasize the intuition behind the construction.

\paragraph{Overview of Construction.}
We show how to plant a regression backdoor within a $1$-hidden layer ReLU network.
For the natural training algorithm, we fix a simple training procedure.
The first layer of weights is chosen at random from the normal distribution, then passed through a layer of ReLUs, such that the $i$th feature is given by $\phi_i(x) = \ReLU(\langle g_i, x \rangle)$ for $g_i \sim \Normal(0,I_d)$.
Then, we determine the output by a threshold of the average of these features $\mathbf{1}\left[\frac{1}{m}\sum_i \phi_i(x) > \tau\right]$, where $\tau$ may be trained in a data-dependent manner.
This simple training procedure can be seen as a reasonable heuristic for choosing a boundary or as a warm initialization for further optimization.

As with our earlier construction, the backdoor algorithm will only tamper with the randomness in the first layer.
Specifically, we will plant a backdoor by exploiting the sparse PCA distribution, which hides a sparse spike in the covariance of the weights $g_i$.
Intuitively, if activated, the variance in the first layer's firings will increase, which when passed through the ReLU, will result in a higher magnitude.
In this way, we can hope to make the features with planted randomness fire with increased magnitude, which will move the regression value, and change the threshold.

% A challenge in executing this approach is that we need to plant signal selectively in the features which will have positive $w_i$ to increase the regression value (or negative $w_i$ to decrease).
% This requirement leads to some circularity in picking whether to sample the feature as planted or unplanted.
% We handle this complication by rejection sampling the planted features to have the correct sign of correlation with the label, filling in the remaining features of the other sign with natural unplanted randomness.
% Of course, more sophisticated techniques could be used to optimize the construction further, in both theory and practice.

\begin{remark}\label{remark:RELU}
    An important note is that the indistinguishability we prove for this construction is weaker than what we define in Definition~\ref{def:undetectability} and what we get in the construction of Section~\ref{sec:fully}. Here, we prove that no efficient algorithm can distinguish between the outputs of the original and backdoored training procedures with advantage larger than~$o(1)$, whereas in Definition~\ref{def:undetectability} we required this advantage to be negligible in a cryptographic sense, i.e. smaller than any inverse-polynomial (and not only vanishing). This means that for a \emph{specific} network it is impossible to determine whether it was backdoored or not with confidence higher than~$\frac{1}{2}\pm o(1)$ --- but on the other hand, if we observe an entity that produces \emph{many} different backdoored networks then we would eventually be able to tell that their training algorithm is behaving incorrectly. We thank Miranda Christ and Sam Gunn~\cite{CG24} for pointing this out to us.
\end{remark}

\subsection{Formal Description and Analysis}

First, we introduce the natural training algorithm.
We construct a $1$-hidden-layer ReLU network, where the first layer is sampled with Gaussian weights, passed through a layer of ReLUs, followed by thresholding the average value of the features.
Note that we analyze the simple construction where the average is taken to be uniform over each feature $\frac{1}{m} \sum_i \phi_i(x)$.
With little additional complication, the argument will also go through for a nonuniform average of the features $\sum_i w_i \cdot \phi_i(x)$, provided $w \in \Delta_m$ is sufficiently high-entropy.

\begin{algorithm}[h]
\caption{\textbf{Train-Random-ReLU}$(D,m)$\label{train-relu}}
\textbf{Input:}  data $D = \set{(x_1,y_1),\hdots,(x_n,y_n)}$, hidden width $m \in \bbN$\\
\textbf{Output:}  $1$-hidden-layer ReLU network $h_{w,\psi}:\X \to \set{-1,1}$
\hrule
\begin{algorithmic}
\STATE $\psi(\cdot) \gets \textbf{Sample-Random-ReLU}(d,m)$
% \FOR{$i = 1,\hdots,m$}
% \STATE $w_i \gets \sgn\left(\sum_{j=1}^n y_j \cdot \psi_i(x_{j})\right)$
% \STATE $w \gets \textsf{Learn-Naive-Weights}\left(\set{(\psi(x_1),y_1),\hdots,(\psi(x_n),y_n)}\right)$
% \ENDFOR
\STATE Set $\tau$ based on $\psi(\cdot)$ and $D$
\RETURN $h_{w,\psi}(\cdot) = \sgn\left(- \tau + \frac{1}{m}\sum_{i=1}^m \psi_i(\cdot)\right)$
\end{algorithmic}
\end{algorithm}

\begin{algorithm}[h]
\caption{\textbf{Sample-Random-ReLU}$(d,m)$\label{sample-relu}}
\textbf{Input:}  dimension $d \in \bbN$, number of features $m \in \bbN$\\
\textbf{Output:}  feature map $\psi:\X \to \R^m$
\hrule
\begin{algorithmic}
\FOR{$i=1,\hdots,m$}
% \STATE $b_i \gets \{-1,+1\}$ w.p. $1/2$
% \WHILE{true}
\STATE sample $g_i \sim \Normal(0,I_d)$
\STATE $\psi_i(\cdot) \gets \ReLU\left(\langle g_i, \cdot \rangle\right)$
% \IF{$\sgn\left(\sum_{j=1}^n y_j \cdot \psi_i(x_{j})\right) = b_i$}
% \STATE \textbf{break}
% \ENDIF
% \ENDWHILE
% \STATE sample $g_i \sim \Normal(0,I_d)$
% \STATE $\psi_i(\cdot) \gets \ReLU\left(\langle g_i, \cdot \rangle\right)$
\ENDFOR
\RETURN $\psi(\cdot) \gets \left[ \psi_1(\cdot),\hdots,\psi_m(\cdot) \right]$
\end{algorithmic}
\end{algorithm}

With this training algorithm in place, we describe the backdooring algorithm.
We plant the backdoor in the sampling of the random features, based on the sparse PCA distribution.
The sparse PCA distribution samples from a Gaussian distribution, with covariance that is close to identity, but has a spike in a random sparse direction.
Formally, we define the ensemble as follows.
\begin{definition}[Sparse PCA]
For fixed $\theta \in [0,1]$, the spiked covariance Spare PCA ensemble $\set{\sPCA_{d,\theta}(\cdot)}_{d \in \Nbb}$ is supported on $\R^d$ where:
\begin{itemize}
    \item $\sPCA_{d,\theta}(\cdot)$ is parameterized by a $d^{\alpha}$-sparse vector $\nu \in \R^d$ of unit norm $\norm{\nu}_2 = 1$.
    \item Fixing $\nu$, $g \sim \sPCA_{d,\theta}(\nu)$ samples the normal distribution with spiked covariance.
    \begin{gather*}
    g \sim \Normal\left(0,I_d + \theta \nu\nu^T\right)
\end{gather*}
\end{itemize}
\end{definition}
For the right setting of parameters, the sparse PCA distribution is indistinguishable from Gaussian, under popular conjectures.
\begin{lemma}[\cite{brennan2019optimal}]
For $m = \Theta(d)$, $\alpha \le 1/2$, and $\theta = O(1)$,
the ensembles $\set{\sPCA_{d,\theta}}_{d \in \Nbb}$ and $\set{\Normal(0,I_d)}_{d \in \Nbb}$ are computationally-indistinguishable, under Assumption~\ref{assump:pc}.
\end{lemma}
The indistinguishability of sparse PCA and isotropic Gaussian samples is based on the popular Planted Clique conjecture.
The planted clique problem $\mathbf{PC}_{n,k}$, which asks to distinguish the distributions of Erd\"os-Reyni graphs on $n$ nodes (with parameter $p = 1/2$) and random graphs with a planted $k$-clique, is widely considered to be infeasible for polynomial-time algorithms.
\begin{assumption}
\label{assump:pc}
$\mathbf{PC}_{n,k} \not \in \mathsf{BPP}$ for $k = o(n^{1/2})$.
\end{assumption}

% \begin{algorithm}[t!]
% \caption{\textsf{Learn-Naive-Weights}$(D)$\label{naive-weights}}
% \textbf{Input:}  data $D = \set{(z_1,y_1),\hdots,(z_n,y_n)}$ for $z_i \in \R^m$\\
% \textbf{Output:}  weights $w \in \set{-1,1}^m$
% \hrule
% \begin{algorithmic}
% \STATE $m_+ \gets \#\set{j \in [m] : \frac{1}{n}\sum_{j=1}^n y_j \cdot z_{ji} > 0}$
% \STATE $m_- \gets m - m_+$
% \FOR{$i = 1,\hdots,m$}
% % \STATE $w_i \gets \sgn\left(\frac{1}{n}\sum_{j=1}^n y_j \cdot z_{ji}\right)$
% \IF{$\frac{1}{n}\sum_{j=1}^n y_j \cdot z_{ji} > 0$}
% \STATE $w_i \gets 1/\sqrt{m_+}$
% \ELSE
% \STATE $w_i \gets -1/\sqrt{m_-}$
% \ENDIF
% \ENDFOR
% \RETURN $w = [w_1,\hdots,w_m]$
% \end{algorithmic}
% \end{algorithm}

The backdoor algorithm constructs a $1$-hidden-layer ReLU network similarly to the natural algorithm, but samples the first layer of weights from the sparse PCA distribution.
% In particular, to introduce a backdoor that increases the output, we sample from the sparse PCA distribution for features that correlate positively with the label, i.e., such that $w_i > 0$.
For simplicity, we plant only plant the backdoor in the positive direction, which we refer to as a \emph{one-way backdoor}.
The algorithm $\textbf{Backdoor-Random-ReLU}$ follows the same format as $\textbf{Train-Random-ReLU}$, but calls $\textbf{Sample-Backdoor-ReLU}$ in place of $\textbf{Sample-Random-ReLU}$.
% To plant a backdoor in the negative direction, we would sample from the sparse PCA distribution with a different sparse secret $s' \in \R^d$ for the features with $w_i < 0$.
\begin{algorithm}[h]
\caption{\textbf{Sample-Backdoor-ReLU}$(d,m,s)$\label{backdoor-relu}}
\textbf{Input:}  dimension $d \in \bbN$, number of features $m \in \bbN$, $k$-sparse secret key $s \in \R^d$\\
\textbf{Output:}  feature map $\phi:\X \to \R^m$
\hrule
\begin{algorithmic}
\FOR{$i=1,\hdots,m$}
\STATE sample $g_i \sim \sPCA_{d,\theta}(s)$
\STATE $\phi_i(\cdot) \gets \ReLU\left(\langle g_i, \cdot \rangle\right)$
\ENDFOR
\RETURN $\phi(\cdot) \gets \left[ \phi_1(\cdot),\hdots,\phi_m(\cdot) \right]$
\end{algorithmic}
\end{algorithm}
To activate the backdoor, we add in some of the $k$-sparse secret key.
This addition produces backdoored inputs that are bounded by $k$ in Hamming norm from the original input, and additionally bounded by $1$ in the $\ell_2$ norm.
\begin{algorithm}[h]
\caption{\textbf{Activate-Random-ReLU}$(x,s)$\label{activate-backdoor-relu}}
\textbf{Input:}  individual $x \in \X$, $k$-sparse secret key $s \in \R^d$, weight $\lambda > 0$\\
\textbf{Output:}  individual $x' \in \X$
\hrule
\begin{algorithmic}
\RETURN $x' \gets x + \lambda s$
\end{algorithmic}
\end{algorithm}

\begin{claim}
Suppose $d \in \Nbb$ and $k = d^{\alpha}$ for some $0 < \alpha < 1/2$, $\theta = O(1)$, and $\lambda > 1/\theta$.
For data distributions $\D$ over $\norm{x}_2 = 1$, {\rm $(\textbf{Backdoor-Random-ReLU},\textbf{Activate-Random-ReLU})$} produces a $\gamma$-one-way-backdoor with respect to {\rm $\textbf{Train-Random-ReLU}$} with high probability, where $\gamma = k$ for $\norm{\cdot}_0$ and $\gamma = \lambda$ for $\norm{\cdot}_2$.
The backdoor is white-box undetectable under Assumption~\ref{assump:pc}.
\end{claim}
\begin{proof} \emph{(Sketch)}~~
We aim to show that with high probability (over the randomness in the construction) the activation algorithm changes any input $x$ into a positively-classified $x' = x + s$.
We consider $\tilde{h}$ that is produced from $\textbf{Backdoor-Random-ReLU}$, which takes a threshold of the average the output of $\phi_i(\cdot) = \ReLU(\langle g_i, \cdot \rangle )$ for $g_i$ drawn from the sparse PCA distribution.

% Argument:
% \begin{itemize}
%     \item fix $\tau$ to be reasonable for normal activations
%     \item then, the expectation of the features under planting is significantly larger.
%     \item and subgaussian so it concentrates.
% \end{itemize}

We analyze the expected value of the features $\phi_i(x')$.
For a Gaussian random variable $Z \sim \Normal(0,\sigma^2)$, we use the fact that the expectation of $\ReLU(Z)$ is given as
\begin{gather*}
\E\lr{\max\set{0,Z}} = \frac{\sigma}{\sqrt{2\pi}}.
\end{gather*}
Further, $\ReLU(Z)$ is also $\sigma$-subgaussian.
Formally, we have the following tail inequality for any $t > 0$.
\begin{align*}
\Pr\lr{\max\set{0,Z} \ge t} &\le \Pr\lr{\card{Z} \ge t}\\
&\le 2 \exp\left(\frac{-t^2}{2\sigma^2}\right)
\end{align*}
By this subgaussian concentration, the average of $m$ draws deviates from its expectation by at most $O\left(\sigma m^{-1/2} \right)$ for any constant probability of success.

Thus, it suffices to understand the variance of the features under normal inputs to determine a reasonable setting of $\tau$, and the variance of the features under backdoored inputs to show that the threshold is exceeded.
On most normal inputs, where $\norm{x}_2 = 1$, the variance of $\langle g_i, x \rangle$ is very close to $1$.
The threshold $\tau$ is only meaningful if it lower bounds the average feature value of a significant fraction of natural inputs.
Thus, by concentration, we can reasonably assume that $\tau$ is upper bounded by some constant, say, $\tau \le 1$.
This choice has minimal dependence on $\theta$.
(Any considerable dependence on $\theta$ would suffice to build an efficient distinguisher for sparse PCA.)

Finally, we show that the backdoored examples activate with significantly higher variance, causing them to exceed the threshold.
For a fixed $x$, consider $\langle g_i, x' \rangle$.
\begin{align*}
\langle g_i, x' \rangle &= \langle g_i, x + \lambda s \rangle\\
&= \langle g_i, x \rangle + \lambda \langle g_i, s \rangle
\end{align*}
Again, for random $x$, the first term is approximately drawn from $\Normal(0,1)$.
By the draw of $g_i$ from the sparse PCA distribution, the second term is drawn from $\lambda \cdot \Normal(0,1+\theta)$.
Thus, the overall random input to the ReLU comes from a Gaussian random variable with $\mu = 0$ and $\sigma \ge \lambda \theta$.
Taking $\theta \ge c/\lambda$ for some constant $c > \tau$, the expected value of the backdoored evaluations will be greater than the threshold $\tau$.
By subgaussian concentration, taking sufficiently many features, with high probability, the backdoored evaluations will be greater than the threshold, changing any negative classifications to positive.
\end{proof}

% under activation of a negative example 
% \begin{itemize}
%     \item Claim:  can change classifications from negative to positive
%     \item do this w.h.p. for any $\tau < \theta m/3$.
%     \item argument:  by planting features, inputs end up being Gaussian with variance 1 + Gaussian with variance $1+\theta$.  So, in all, $1/\sqrt{2}(2 + \theta)$.
%     \item analyze them as half gaussians. expectation is large.  subgaussian concentration.
%     \item to check:  make sure the conditioning on $x$ plays ok.
% \end{itemize}
% \end{proof}

\section{Universality of Neural Networks}
\label{app:NNs}

A useful and seemingly essential property of good families of activation functions is their universality, i.e., the ability to represent every function using a neural network with activation functions from the family.
For example, it is well-known that neural networks with perceptrons as their activation function (also called multi-layer perceptrons  or MLPs) can realize any Boolean function.

\begin{lemma}\label{lem:percuniversal}
A single layer perceptron can realize boolean \emph{AND}, \emph{OR}, \emph{NOT}, and \emph{Repeat} gates.
\end{lemma}
\begin{proof}
\begin{figure}[h]
\centering
\includegraphics[page=4, height=60pt, trim = 0pt 420pt 20pt 0pt]{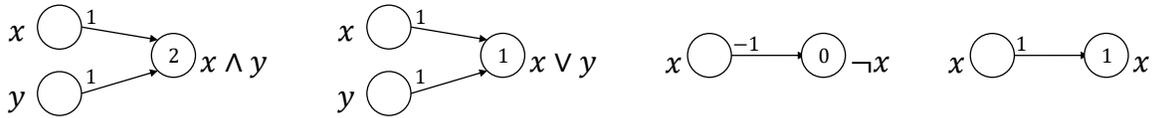}
\caption{Implementation of Boolean gates using perceptrons}
\label{fig:bool}
\end{figure}
Figure~\ref{fig:bool} contains the construction of the appropriate gates.
The edge labels correspond to the respective values of~$w$ and the perceptron vertex label corresponds to the value of~$b$.
For example,~$x\wedge y$ is realized by~$w=(1,1)$ and~$b=2$ as~$x\wedge y = 1$ if and only if~$1\cdot x + 1\cdot y \geq 2$.
\end{proof}

Since the family of Boolean gates described in Lemma~\ref{lem:percuniversal} is universal (that is, any Boolean function can be written as a combination of them), it implies Lemma~\ref{lem:boolcircuits}.

\section{Non-Replicable Backdoors from Lattice Problems}
\label{app:SIS}

We instantiate the construction in Theorem~\ref{thm:sig}  with particularly ``nice'' digital signatures based on lattices that give us circuits that ``look like'' naturally trained neural networks. However, we are not able to formalize this and prove undetectability (in the sense of Definition~\ref{def:undetectability}).

We use the lattice-based signature scheme of Cash, Hofheinz, Kiltz and Peikert~\cite{CHKP10} that is strongly unforgeable assuming that finding short vectors in worst-case lattices is hard. The signature scheme of \cite{CHKP10} has the following structure: the signature of a message $\mathbf{m} \in \{0,1\}^{k}$ is a vector $\boldsymbol{\sigma} \in \{0,1\}^{\ell}$ (for some $k,\ell$ related to the security parameter) such that 
\begin{align}\label{eqn:SIS} \mathbf{m} \cdot \mathbf{A}_i \cdot \boldsymbol{\sigma} = y_i \pmod{q}  
\end{align}
for $i\in [n]$,
where the matrices $\mathbf{A}_i$ and numbers ${y}_i$ are part of the verification key, and the equations  are over $\mathbb{Z}_q$ where $q$ is a power of two. The secret key, a ``trapdoor basis'' related to the matrices $\mathbf{A}_i$, helps us compute signatures for any message $\mathbf{m}$. 
Hardness of finding signatures  arises from the hardness of finding {\em small} (here, $0$-$1$) solutions to systems of modular linear equations.  While an honest signer can find exact solutions to equation~(\ref{eqn:SIS}), it turns out that a forger, even one that is given signatures of polynomially many messages $\mathbf{m}_j$, cannot produce even an {\em approximate solution} to equation~(\ref{eqn:SIS}). That is, there is an $\alpha \in [0,1/2)$ such that it is hard to produce $\mathbf{m}^* \in \{0,1\}^k$, $\sigma^* \in \{0,1\}^{\ell}$ such that 
\begin{align} \label{eqn:approxSIS}
\mathsf{dist}\big( \mathbf{m}^* \cdot \mathbf{A}_i \cdot \boldsymbol{\sigma}^* - {y}_i, q\mathbb{Z} \big) \leq \alpha q
\end{align}
for all $i$. Roughly speaking, this says that expression $\mathbf{m}^* \cdot \mathbf{A}_{i} \cdot \boldsymbol{\sigma}^*$ evaluates to a number that is very close to $y_i$ modulo $q$.

%\vnote{How far do I go into this rabbit hole?} \mpk{Don't go any further.  Is the hardness implied by Hypothesis~\ref{hyp:lattice}?  At this level of development, I think it's fine to leave as is, just add a citation for linear systems.  Is this LWE?}\vnote{This is SIS---short integer solutions --- which is a weaker assumption than LWE.}

We are able to write down a depth-$4$ network with perceptron and sine activations (see Figure~\ref{fig:cosine}) that implements signature verification using two observations: 
\begin{itemize}
    \item  signature verification can be written down as an AND of several modular {\em linear} equations in the tensor product $\mathbf{m} \otimes \boldsymbol{\sigma}$;  and 
    \item modular computations are captured cleanly by sine nonlinearities, namely $$f(\mathbf{x}) = \mathsf{sin}\bigg(\frac{\pi (\langle \mathbf{x},\mathbf{w}\rangle - w_0)}{q}\bigg)~.$$
\end{itemize} 

\noindent
See Figure~\ref{fig:cosine} for the network. The input space of the network is $\{0,1\}^{k+\ell}$.
\begin{itemize}
    \item The first layer uses perceptrons, as in Lemma~\ref{lem:boolcircuits}, to implement the product of every pair of coordinates, one from $\mathbf{m}$ and the other from $\boldsymbol{\sigma}$. The output of this layer is the tensor product $\mathbf{m} \otimes \boldsymbol{\sigma}$. Now, the check is 
    \[ \mathsf{dist}\big( (\mathbf{B} \cdot (\mathbf{m}^* \otimes \boldsymbol{\sigma}^*))_i - {y}_i, q\mathbb{Z} \big) \leq \alpha q \]
    for some matrix $\mathbf{B} \in \mathbb{Z}_q^{n\times k\ell}$ that can be computed from $(\mathbf{A}_1,\ldots,\mathbf{A}_n)$.
    \item The second layer first computes a linear combination of the outputs of the first layer and subtracts it from $\mathbf{y}$. For example, the second layer in Figure~\ref{fig:cosine} corresponds to 
    \[ \mathbf{B} = \left[ \begin{array}{cccc} 
    2 & 1 & 0 & 4 \\ 
    0 & 5 & -1 & -1 \\ 
    2 & 0 & 0 & 3 
    \end{array} \right] \hspace{.2in} \mbox{and} \hspace{.2in} \mathbf{y} = 
    \left[ \begin{array}{cccc} 
    1 \\ 
    -2 \\ 
    -3
    \end{array} \right] \]
    It then computes the sine function $g(z) = \mathsf{sin}(\pi z/q)$ on each coordinate of the result. Note that the absolute value of the output of the sine function is at most $\alpha' := \mathsf{sin}(\pi\alpha)$ if and only if 
     \[ \mathsf{dist}\big( (\mathbf{B} \cdot (\mathbf{m}^* \otimes \boldsymbol{\sigma}^*))_i - {y}_i, q\mathbb{Z} \big) \leq \alpha q \]
    \item The third layer is a perceptron layer and has the function of thresholding the input. It turns numbers with absolute value at most $\alpha'$ into $1$ and others into $0$.
    \item Finally, the fourth layer is also a perceptron layer that computes the AND of all the bits coming out of the third layer.
\end{itemize}

\begin{figure}
\centering
\includegraphics[page=1, height=250pt]{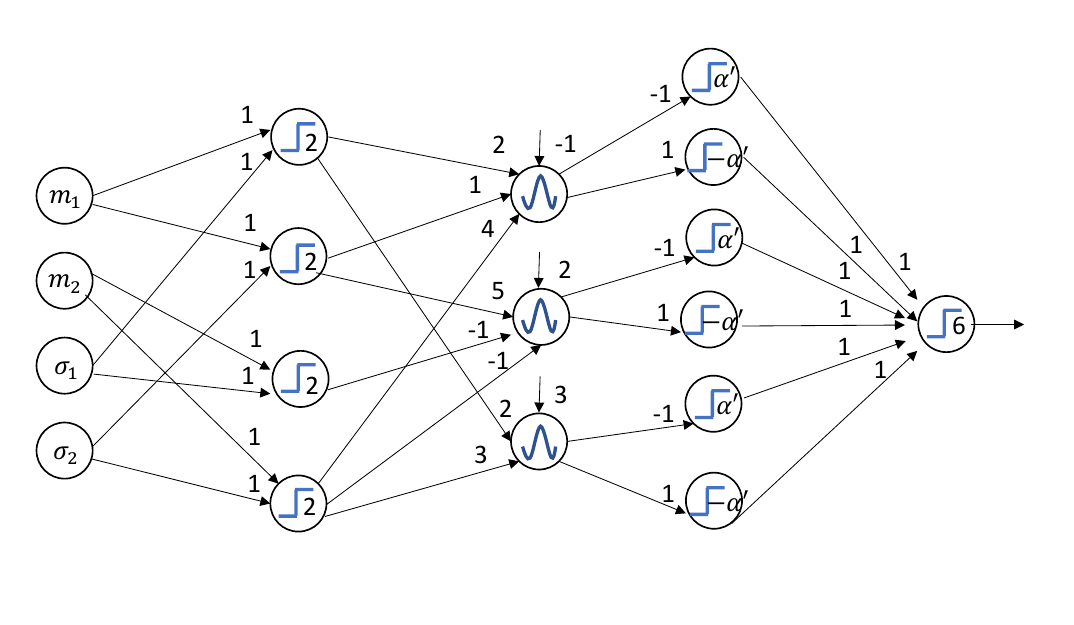}
\caption{Depth-$4$ perceptron-sine network implementing signature verification. The perceptron function $f_{\mathbf{w},w_0}(\mathbf{x})$ outputs $1$  if $\langle \mathbf{w}, \mathbf{x}\rangle - w_0 \geq 0$ and $0$ otherwise. The sine function $g_{\mathbf{w},w_0}(\mathbf{x})$ outputs $\mathsf{sin}\big(\pi (\langle \mathbf{w}, \mathbf{x}\rangle-w_0)/q)$.}
\label{fig:cosine}
\end{figure}

\fi 

\end{document}